\documentclass[10pt,twocolumn,letterpaper]{article}

\usepackage[pagenumbers]{cvpr} 

\usepackage{xspace}
\usepackage{booktabs}
\usepackage{multirow}
\usepackage{makecell}
\usepackage{amsmath,amssymb,amsthm,mathtools}
\usepackage{tikz}
\usetikzlibrary{arrows.meta,positioning,fit,calc,backgrounds,shapes.geometric}
\usepackage[table]{xcolor}
\definecolor{qGT}{HTML}{1B9E3A}\definecolor{qDet}{HTML}{D62728}\definecolor{qFused}{HTML}{444444}
\definecolor{qEgo}{HTML}{999999}\definecolor{qBlue}{HTML}{2A78D6}\definecolor{qOrange}{HTML}{EB6834}\definecolor{qGreen}{HTML}{57975B}
\DeclareRobustCommand{\lgd}[1]{\tikz[baseline=-0.25ex]\draw[#1] (0,-0.45ex) rectangle (0.9em,0.75ex);}
\DeclareRobustCommand{\lgdfill}[1]{\tikz[baseline=-0.25ex]\fill[#1] (0,-0.45ex) rectangle (0.9em,0.75ex);}
\DeclareRobustCommand{\lgdtri}[1]{\tikz[baseline=-0.25ex]\node[regular polygon,regular polygon sides=3,fill=#1,draw=black,line width=0.3pt,inner sep=0pt,minimum size=1.3ex] at (0,0.1ex) {};}
\usepackage{subcaption}
\usepackage{newtxtext}
\usepackage{threeparttable}

\definecolor{Gray}{gray}{0.93}
\definecolor{OursBlue}{RGB}{226,238,250}
\definecolor{jqblue}{HTML}{2a78d6}
\definecolor{jqorange}{HTML}{eb6834}
\definecolor{jqgreen}{HTML}{57975b}

\newcommand{\system}{\textsc{JigsawComm}\xspace}
\DeclareMathOperator*{\argmax}{arg\,max}

\newtheorem{proposition}{Proposition}

\theoremstyle{definition}

\usepackage{algorithm}
\usepackage[noend]{algpseudocode}
\algrenewcommand\algorithmiccomment[1]{\hfill{\color{gray}\small$\triangleright$ #1}}

\definecolor{cvprblue}{rgb}{0.21,0.49,0.74}
\usepackage[pagebackref,breaklinks,colorlinks,allcolors=cvprblue]{hyperref}

\def\confName{CVPR}
\def\confYear{2027}

\title{Assemble the Puzzle: Cross-Agent Utility-Aware Feature Scheduling for Bandwidth-Limited Cooperative Perception}

\author{Chenyi Wang\textsuperscript{\rm 1}\quad Muhammad Waqar\textsuperscript{\rm 2}\quad Wujie Wen\textsuperscript{\rm 2}\quad Ming F. Li\textsuperscript{\rm 1}\\
{\textsuperscript{\rm 1}University of Arizona\quad \textsuperscript{\rm 2}North Carolina State University}\\
{\tt\small \{chenyiw, lim\}@arizona.edu\quad \{mwaqar, wwen2\}@ncsu.edu}
}

\begin{document}
\maketitle
\begin{abstract}
Cooperative perception (CP) lets connected vehicles and roadside units detect objects beyond the inherent range limitation and occlusion of single-agent systems, yet its practicality is constrained by limited Vehicle-to-Everything (V2X) communication bandwidth. 
Existing methods improve bandwidth efficiency by letting collaborating agents send only selected feature cells driven by local detection confidence, serving a centralized ego agent's request. They overlook realistic communication scheduling and neglect cross-agent redundancy.  
We present \system, a decentralized utility-aware feature scheduling framework that learns to ``assemble the puzzle'' of multi-agent features.
Choosing who sends each region requires a utility that measures what each agent's features contribute to the collaboration. 
%
%
\system incorporates a lightweight feature utility estimator, learned without utility labels from the cooperative model's own leave-one-agent-out counterfactuals, which aligns utilities onto a single comparable scale across agents. 
From one broadcast of compact utility maps, every agent ranks the observers of each cell and sends the cells it observes best to maximize coverage and, budget permitting, second best for accuracy refinement. 
Each agent also broadcasts its onboard detector's lightweight bounding box list, enabling an intermediate-late hybrid fusion with our feature-anchored box consensus, which synergistically enhances the final perception results with all available observations. 
%
%
%
With every transmitted byte counted, \system at 0.45\,KB per vehicle outperforms all baselines on the real-world V2X-Real benchmark, and under measured C-V2X traces on CooperScene it retains 74.5 of its 79.8 AP@0.5, at least 15.5 points above every prior art baseline.
\end{abstract}\vspace{-8pt}

\section{Introduction}
\label{sec:intro}
\vspace{-2pt}
Multi-agent cooperative perception (CP) lets connected road agents, such as connected and autonomous vehicles (CAVs) and roadside units (RSUs), exchange sensor information over V2X communication~\cite{cpm} and detect objects beyond their own sensing range and behind occlusions~\cite{cp_benefit,xu2022opv2v,xiang2024v2xreal}. Nevertheless, practical deployment of CP faces severe bandwidth constraints. The V2X spectrum is narrow~\cite{FCC_DSRC,CV2X_ETSI_EN_303_613} and the commercial C-V2X radios of the real-world benchmark CooperScene~\cite{wu2026cooperscene} carry about 0.7\,Mbps per link while all agents broadcast; even the 1.6\,Mbps that its benchmark assumes leaves only 20\,KB per vehicle per frame at 10\,Hz. Intermediate fusion, in which agents exchange and fuse Bird's-Eye-View (BEV) features, has become the dominant paradigm because it outperforms sharing raw sensor data while achieving collaborative enhancement~\cite{f_cooper,xu2022opv2v,hu2022where2comm,xu2022v2x}. Yet a dense BEV map takes several to tens of megabytes~\cite{f_cooper}, exceeding the available bandwidth.\looseness=-1

Prior work introduced feature selection methods, which have improved the communication efficiency of intermediate fusion. Where2comm~\cite{hu2022where2comm} showed that sending only the cells a sender is confident about cuts messages by orders of magnitude, and later methods refine this design with request/demand maps or information scores~\cite{xu2025cosdh,tang2025cost,hu2024communication,wei2025infocom}. 
Such methods are \textit{receiver-centric}~\cite{hu2024communication,xu2025cosdh}. While they account for what the receiver already has, they require an extra round of request maps, each as large as a dense feature mask and specific to one centralized receiver. Also, each sender selects sending cells comparing its own view with the centralized request, so neighbors that observe the same object all send it, which causes \emph{cross-agent redundancy}. 
The underlying challenge to resolve such redundancy is estimating the marginal utility of each agent's feature, without central coordination. 
Confidence is the right utility for a locally detected box, directly signaling object presence, but the cooperative value of a feature depends on what the neighbors send. A feature cell corresponding to a locally confident detection of an object contributes little when a neighbor already sees it better, while a less confident view from another side can still refine it. 
Every region therefore raises two questions that compare agents with one another: \emph{who} should send it, and \emph{whether} a second view is worth its bits.

We answer both with \emph{utility-aware scheduling}, in which \system learns to ``assemble the jigsaw puzzle'' of multi-agent features (Fig.~\ref{fig:overview}). Each agent broadcasts a compact \emph{utility map} as meta information, which holds the quantized utilities and entropy-coded positions of its $K$ most valuable cells. From the same maps, every agent ranks the observers of each cell and sends the cells where it ranks first, which \emph{tile the scene without overlap} for coverage, and, budget permitting, those where it ranks second, which \emph{refine} the objects that two agents see. Because every agent runs the same deterministic rule on the same maps, the schedule only needs one broadcast of utility maps before the features, with no coordinator and no receiver feedback. Unlike receiver-centric requests, this \emph{sender-centric} schedule lets one broadcast serve all receivers, saving bandwidth. 

The key enabler of this scheduling policy is a learned utility. Ranking compares one agent's feature with another's, so the utility must be \emph{comparable} across agents and score what a feature adds to the others'. With the cooperative model trained through a differentiable scheduler~\cite{jang2016categorical} to fuse and decode from sparse, assembled feature input, \system learns a \textit{utility estimator} on the frozen model without labels. It predicts a residual on the confidence logit, supervised by a leave-one-agent-out counterfactual: how much the fused prediction changes when one agent withholds its map while the others transmit. The learned utility therefore decreases when other agents already supply the same evidence, which is the redundancy that confidence values do not capture. 

\system's meta information broadcast also carries each agent's lightweight \emph{box list}, the bounding boxes its onboard detector keeps after local non-maximum suppression (NMS), which complements the feature-level cooperation through \emph{feature-anchored box consensus}. Box lists alone lack the semantic evidence to resolve conflicting boxes, and late-fusing them by cross-agent NMS is winner-takes-all, discarding less-confident yet potentially complementary observations without synergistic enhancement. \system uses the cooperative results derived from fused features as anchors, which carry the features' evidence. Each object's localization is then informed by \textit{all} its observations including the box list, by moving the feature-anchored detection to the score-weighted consensus center of every box overlapping it.
Unlike prior intermediate-late hybrid fusion methods, which send boxes only to correct poses~\cite{lu2023robust} or to select features and compensate their delay~\cite{wei2023cobevflow}, realign dense outputs with a learned module~\cite{chen2026cora}, or send dense outputs of the cooperative head and let one box win~\cite{xu2025cosdh}, our feature-anchored box consensus synergistically refines the final fused boxes without dismissing transmitted information, retaining more collaborative benefits without learned parameters.

We evaluate \system on two challenging real-world benchmarks, V2X-Real~\cite{xiang2024v2xreal} and CooperScene~\cite{wu2026cooperscene}, accounting for every transmitted byte, from utility maps and box lists to features and poses. On V2X-Real, \system at 0.45\,KB per vehicle per frame outperforms every baseline at all three IoU thresholds, including methods that send three to four orders of magnitude more data (Fig.~\ref{fig:teaser}). On CooperScene, \system is the most accurate method at AP@0.3 and AP@0.5 from 0.40\,KB and up, and at 2.07\,KB comes within 0.2 AP@0.7 of the best method, which sends 83 times more data. Through CooperScene's measured real-world C-V2X throughput traces, it retains 74.5 of its 79.8 AP@0.5 and leads every prior art baseline evaluated on the link by at least 15.5. Our contributions are:\looseness=-1

\begin{figure}[t]
  \centering
  \includegraphics[width=\linewidth]{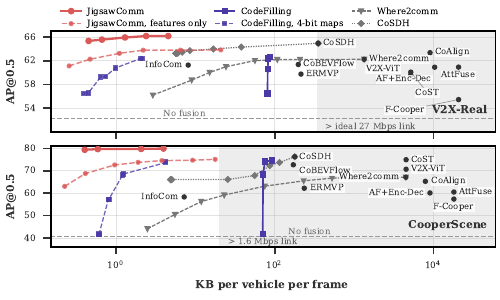}
  \vspace{-22pt}
  \caption{AP@0.5 versus data sent per vehicle per frame on V2X-Real~\cite{xiang2024v2xreal} (top) and CooperScene~\cite{wu2026cooperscene} (bottom). \system matches or exceeds methods that send orders of magnitude more data. Dashed purple: CodeFilling with \system's 4-bit map coding. Shaded: beyond what one link carries per frame at 10\,Hz, for the maximum theoretical 27\,Mbps V2X link (top)~\cite{ligo2019spectrum}, and for CooperScene's benchmark rate of 1.6\,Mbps (bottom)~\cite{wu2026cooperscene}. }
  \label{fig:teaser}
  \vspace{-15pt}
\end{figure}

\begin{itemize}
  \item We present \system, a cross-agent \textit{utility-aware feature scheduling framework} for communication-efficient CP. By exchanging compact utility maps, each agent derives a unified transmission schedule where each cell is sent by its best observer for coverage and, budget permitting, its second best for refinement (\cref{sec:method}).
  \item We \textit{learn cross-agent utilities without manual labels}: an estimator, fitted to the leave-one-agent-out counterfactuals of a cooperative model trained through a differentiable scheduler, scores what a feature adds to the other agents' features (\cref{sec:learning}).
  \item We propose \textit{feature-anchored box consensus}, a learning-free intermediate-late hybrid fusion that synergistically merges the onboard detectors' lightweight box lists with the fused detections by score-weighted center refinement, so that detection coverage is extended and localization of each object is informed by all its observations (\cref{sec:boxes}).
  \item We extensively evaluate \system with full transmission accounting on two real-world datasets. Under measured C-V2X throughput, \textit{\system sends up to three orders of magnitude less data than dense and confidence-based methods while exceeding their accuracy} (\cref{sec:experiments}). Code is available at \url{https://github.com/WiSeR-Lab/JigsawComm} for public use and validation.
\end{itemize}

\vspace{-4pt}
\section{Related Work}
\label{sec:related}
\vspace{-4pt}
\begin{figure*}[t]
  \centering
  \includegraphics[width=\textwidth]{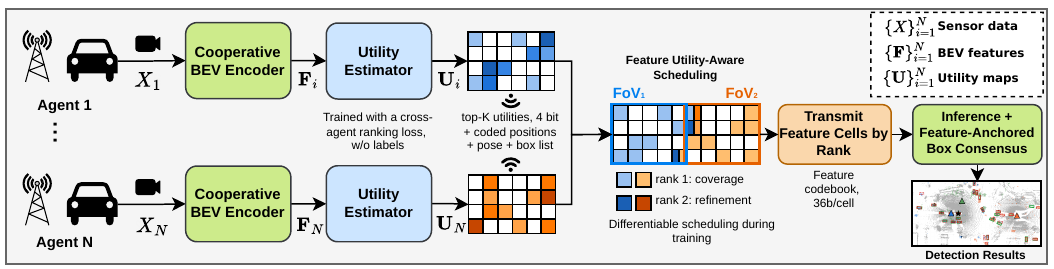}
  \vspace{-20pt}
  \caption{Overview of \system. Each agent encodes its point cloud $X_i$ into BEV features $\mathbf{F}_i$, estimates a utility map $\mathbf{U}_i$, and broadcasts its top-$K$ utilities with coded positions, its pose and its box list. From the same maps, every agent ranks the observers of each cell and derives the unified feature transmission schedule, sending codebook features for the cells where it ranks first (coverage) or second (refinement). Receivers fuse the features, detect, and merge the result with the received box lists via feature-anchored box consensus.}
  \label{fig:overview}
  \vspace{-12pt}
\end{figure*}

\noindent\textbf{Cooperative perception (CP).}
Depending on the stage at which information is exchanged and fused, CP can be broadly categorized into three paradigms.
Early-fusion CP exchanges raw sensor data~\cite{zhang2021emp,zhang2023robust}, at a bandwidth far beyond V2X links~\cite{f_cooper}. Studies in CP have converged on intermediate fusion, which exchanges BEV features and fuses them by max-out~\cite{f_cooper,hu2024communication,xu2025cosdh,wei2023cobevflow}, attention~\cite{xu2022opv2v,xu2022v2x,xu2022cobevt,li2023v2vam} or graph operators~\cite{wang2020v2vnet,li2021disconet}. 
Yet, recent real-world benchmarks~\cite{xiang2024v2xreal,wu2026cooperscene} expose a wide gap between what these methods send and what commercial V2X radios carry. 
Late fusion exchanges local detections~\cite{rauch2012car2x}, which is lightweight but alone lacks the rich feature-space semantics for synergistic accuracy refinement~\cite{xu2022opv2v}. At budgets of a few kilobytes, however, a complete box list costs only 0.2\,KB. Therefore, recent CP designs have started to send detections alongside features to complement intermediate fusion~\cite{lu2023robust, wei2023cobevflow, xu2025cosdh, chen2026cora} (Sec.~\ref{sec:boxes}). Nevertheless, instead of using a winner-takes-all approach (\eg cross-agent NMS~\cite{xu2025cosdh}), which wastes perceptually valuable information already transmitted, \system introduces a feature-anchored hybrid fusion scheme that refines perception results considering all transmitted observations, without learned parameters.\looseness=-1

\noindent\textbf{Communication-efficient CP.}
\emph{Message representation} methods compress the transmitted data with encoder--decoder bottlenecks~\cite{xu2022opv2v,xu2025cosdh}, feature sampling~\cite{Zhang_2024_CVPR}, feature decoupling~\cite{yang2023what2comm}, multi-resolution coding~\cite{wang2023umc}, codebooks~\cite{hu2024communication,zhao2025quantv2x} or generative reconstruction~\cite{wang2023core,mao2025diffcp,zhou2025pragmatic}, which are orthogonal to deciding \emph{which} cells to send.
\emph{Message selection} methods prune the content. When2com~\cite{liu2020when2com} and Who2com~\cite{liu2020who2com} select partners, which saves little on a broadcast channel (the standard mode of V2X operation). Where2comm~\cite{hu2022where2comm} selects cells by detection confidence. How2comm~\cite{yang2023how2comm}, CoSDH~\cite{xu2025cosdh}, CoST~\cite{tang2025cost} and InfoCom~\cite{wei2025infocom} extend it with mutual-information filtering, supply--demand matching, spatiotemporal fusion and an information bottleneck, WhisperNet~\cite{DBLP:conf/cvpr/ChenZZ26} adds channel-wise importance, and CodeFilling~\cite{hu2024communication} fills each receiver's information gaps. Methods that avoid redundant content rely on receiver-centric meta-messages: request/demand maps~\cite{xu2025cosdh,hu2022where2comm} or information-score maps~\cite{hu2024communication} derived from confidence, which are as large as the receiver's BEV map and specific to one receiver. Geometric schemes select by position, field of view or sensor coverage~\cite{zhang2021emp,autocast,wang2024edge}. \system instead has every sender broadcast one compact feature utility map with its local detection box list, which fits a standardized sensor data sharing message~\cite{J3224_202208_BSM_SDSM}, and all agents derive from these maps the same schedule for the heavier feature transmission by comparing sender utilities directly.\looseness=-1

\noindent\textbf{Counterfactual valuation.} Measuring a source's value by the change in output when it is removed has been studied in occlusion sensitivity in vision~\cite{zeiler2014visualizing} and agent relevance estimation for driving, which masks surrounding agents from a policy's input~\cite{bosio2025rdar}. Shapley values~\cite{ghorbani2019datashapley} average it over all subsets at exponential cost, and learned token pruning~\cite{rao2021dynamicvit} selects within one image. \system uses the leave-one-agent-out effect of the cooperative model as a label-free training target, at $N{+}1$ forward passes per scene in training and none at deployment, to select \emph{across} agents.

\vspace{-4pt}
\section{Method}
\label{sec:method}
\vspace{-4pt}
Given each agent's sensor input and a per-frame byte budget, \system decides which BEV cells every agent sends and fuses what each receiver gets into detections (Fig.~\ref{fig:overview}). A scheduling rule assigns each cell to at most two of its most valuable observers from one broadcast of utility maps (Sec.~\ref{sec:assembly}). The utilities are learned without labels on one scale across agents (Sec.~\ref{sec:learning}), and box lists in the same broadcast are merged with the fused detections (Sec.~\ref{sec:boxes}).

\vspace{-2pt}
\subsection{Problem Setup}
\label{sec:setup}\vspace{-2pt}
Consider $N$ agents, where agent $i$ encodes its point cloud $X_i$ into a BEV feature map $\mathbf{F}_i=\Phi_{\text{enc}}(X_i)\in\mathbb{R}^{C\times H\times W}$ in its own frame, with cell features $\mathbf{f}_i^l$ at locations $l\in\Omega$. Instead of the dense map, it transmits a sparse subset $\mathbf{M}_i\odot\mathbf{F}_i$, where $\mathbf{M}_i\in\{0,1\}^{H\times W}$ is a binary mask. Each receiver warps the received features into its own frame with the senders' poses, fuses them with its own map and decodes the detections $\hat{Y}=\Phi_{\text{dec}}(\Phi_{\text{fuse}}(\{\mathbf{M}_i\odot\mathbf{F}_i\}_i))$. With model parameters $\theta$, the goal is to maximize a detection metric $\mathcal P$ under a budget of $B$ bytes per agent per frame:\vspace{-2pt}
\begin{equation}
  \max_{\theta,\{\mathbf{M}_i\}}\;\mathcal{P}\big(\hat{Y},Y\big)
  \quad\text{s.t.}\quad \mathrm{bytes}_i\le B\;\;\forall i.
  \label{eq:objective}
\end{equation}

\noindent\textbf{Transmission accounting.} V2X is a broadcast medium~\cite{FCC_DSRC,CV2X_ETSI_EN_303_613}: one message reaches every neighbor, and every vehicle is a receiver. We therefore bill each agent once per frame for its whole broadcast, and $\mathrm{bytes}_i$ counts every field: the compressed features, the positions of the transmitted cells or the utility maps they are derived from, any other coordination metadata, box lists, and the pose. We report \emph{KB per vehicle per frame}, averaged over vehicles.

\noindent\textbf{Challenges.} (\texttt{C1}) \emph{Cross-agent coupling}: the value of sending $\mathbf{f}_i^l$ depends on which other agents send cell $l$, so the decisions must be coordinated without exchanging the features themselves. (\texttt{C2}) \emph{No ground-truth utility}: no label says how much a cell from an agent contributes to the cooperative result. (\texttt{C3}) \emph{Discrete decisions}: $\mathbf{M}_i$ is binary, which blocks gradient-based training through the task loss.

\vspace{-2pt}
\subsection{Utility-Guided Jigsaw Assembly}
\label{sec:assembly}\vspace{-2pt}

Whether agent $i$ should send cell $l$ depends on who else observes it (\texttt{C1}), and the features themselves are too large to exchange for this decision. \system therefore compares scalar utilities instead of features. Assume for now that each agent has a utility $u_i^l\in[0,1]$ for each of its cells, on one scale across agents (Sec.~\ref{sec:learning}). Because every agent receives the same utility maps, every agent derives the same assignment, in three steps per frame with one signaling broadcast and no feedback (Fig.~\ref{fig:overview}):
\begin{enumerate}
  \item \textbf{Publish.} Agent $i$ selects its candidates $\mathcal{K}_i$, its top $K$ cells by utility inside at least one other agent's BEV range. Each agent sets $K$, a fraction of its map's cells, so that its message fits the budget $B$, and broadcasts its quantized utilities, the indexing positions of the utilities, its pose and its box list (Sec.~\ref{sec:boxes}).
  \item \textbf{Rank.} Every agent warps published maps into its own grid\footnote{From the published maps, agents can also register each other's poses to align utilities before warping, at no added data (Sec.~\ref{sec:robust}).} and ranks each candidate $l\in\mathcal{K}_i$ among the agents that publish the same cell, breaking ties by agent index:
  \begin{equation}
    r_i^l=1+\big|\{j\ne i:l\in\mathcal{K}_j, \big(u_j^l>u_i^l \lor (u_j^l=u_i^l \land j<i)\big) \}\big|.
    \label{eq:rank}
  \end{equation}
  \item \textbf{Send.} Agent $i$ transmits the feature cells with $r_i^l\le R$.
\end{enumerate}
Rank-1 cells tile the observed scene without overlap and provide coverage; rank-2 cells add a second view where two agents see the same region and provide refinement. We use $R{=}2$ because a second view, often from another side, adds evidence the first lacks, while a third mostly repeats the first two. Appendix~\ref{app:utility} models these diminishing returns, and Sec.~\ref{sec:ablation} measures them. Algorithms~\ref{alg:protocol} and~\ref{alg:train} in Appendix~\ref{app:alg} list the full protocol and training for \system.

Because high-utility candidates cluster around objects, we entropy-code their positions with a context-adaptive coder designed for bi-level images~\cite{krichevsky1981performance,jbig1993}, which costs a few bits per candidate instead of a full index. The transmitted cells need no positions at all, since every receiver derives them from the utility maps, and carry residual multi-codebook indices~\cite{hu2024communication}. Unlike the receiver-centric request maps of Where2comm and CoSDH~\cite{hu2022where2comm,xu2025cosdh}, which a receiver returns to its senders, the utility map is broadcast once and serves every receiver. Each agent sends at most its $K$ candidates, and at most $R$ agents send any cell, however many observe it. An agent that misses another's map ranks itself against fewer contenders, so it may send a duplicate but never skips a cell it should cover (Appendix~\ref{app:proofs}).\looseness=-1

\vspace{-2pt}
\subsection{Learning Utilities on One Scale}
\label{sec:learning}\vspace{-2pt}

For the utility-aware ranking and scheduling in \cref{eq:rank} to function correctly, the utilities must be directly comparable across agents and score what a feature adds to the others'. The local detection confidence $c_i^l$ of a shared single-agent head is on one scale by construction and is a strong baseline (Sec.~\ref{sec:ablation}). While confidence measures object presence (Sec.~\ref{sec:boxes}), it does not measure what an agent's feature cell \emph{adds} to what its neighbors send: a confident view of a vehicle that a neighbor sees better adds little. To measure such utility without ground-truth labels (\texttt{C2}), we use the cooperative model itself as the teacher, in two stages: a perception model trained on sparse, assembled input, and an estimator fitted to its counterfactuals.

\noindent\textbf{Transmission-aware perception model.} A model trained on dense, overlapping maps rarely sees a few cells per agent assembled from different senders, so we train the encoder, the fusion and the heads on such assemblies. Fusion is channel-wise max pooling, linear in the number of agents~\cite{f_cooper,hu2024communication,zimmer2024tumtrafv2x,xu2025cosdh}. Because the binary mask blocks gradients (\texttt{C3}), training replaces the assignment with a differentiable cross-agent competition: a $1\times1$ gate head scores each cell with $s_i^l$, and at every cell the agents, the receiver included, compete through a straight-through Gumbel-softmax~\cite{jang2016categorical,liu24differentiable_comb}:
\begin{equation}
  m_i^l=\mathbb{I}[s_i^l\ge\tau]\cdot\mathrm{OneHot}_i\Big(\argmax_j\big(s_j^l/\eta+\xi_j^l\big)\Big),
  \label{eq:gumbel}
  \vspace{-6pt}
\end{equation}
with $\xi_j^l\sim\mathrm{Gumbel}(0,1)$, an annealed temperature $\eta$ and a learned threshold $\tau$. At most one agent, the Gumbel winner, sends a cell, and only if its score clears $\tau$; the backward pass uses $\sigma((s_i^l-\tau)/\eta)\cdot\mathrm{softmax}_i(s^l/\eta+\xi^l)$, so the task gradient reaches a feature only through its competition with the cell's co-observers. A rate term keeps the transmitted fraction of each map, $\bar\rho$, near a target $\rho^\star$:
\begin{equation}
  \mathcal{L}=\mathcal{L}_{\text{det}}+\lambda_r\,\big|\bar{\rho}-\rho^\star\big|.
  \label{eq:loss}
  \vspace{-4pt}
\end{equation}
The codebook~\cite{hu2024communication} is then fitted with the model frozen. The gate head only shapes the perception model: it saturates into near-binary scores that cannot rank cells at small budgets (Appendix~\ref{app:e2ehead}), so deployment uses a dedicated estimator.

\noindent\textbf{Counterfactual utility estimator.} Let $p(\mathcal{A})$ be the fused detection probabilities at a receiver when the agents in $\mathcal{A}$ send their maps. The target for agent $i$ at cell $l$ is the leave-one-agent-out change\vspace{-2pt}
\begin{equation}
  \delta_i^l=\max_{o}\big|\,p_o^l(\mathcal{A})-p_o^l(\mathcal{A}\setminus\{i\})\,\big|
  \label{eq:counterfactual}
  \vspace{-2pt}
\end{equation}
over anchors $o$. Under max fusion, $\delta_i^l$ is small wherever other agents' features match or dominate $i$'s, so the target captures redundancy as the fusion model experiences it rather than as a geometric rule assumes. The estimator never sees which agent receives, so its utility does not depend on the receiver. Since \cref{eq:rank} uses only the \emph{order} of utilities, we fit it with a weighted pairwise logistic loss~\cite{burges2005ranknet}, $w_{ab}\log\!\big(1+e^{-\operatorname{sign}(\delta_a-\delta_b)(z_a-z_b)}\big)$ on logits $z$ with $w_{ab}\propto|\delta_a-\delta_b|$. \emph{Cross-agent} pairs, which compare two agents at a cell they both observe, put all agents on one scale, and \emph{within-agent} pairs near the top-$K$ cutoff decide the candidates. The estimator uses only the agent's own data and predicts a residual on the confidence logit,
\begin{equation}
  u_i^l=\sigma\!\big(\mathrm{logit}(c_i^l)+h_\phi(\mathbf{F}_i,\mathbf{c}_i,\mathbf{d}_i)^l\big),
  \label{eq:residual}
\end{equation}
where $\mathbf{c}_i$ is the confidence map, $\mathbf{d}_i$ the point density and $h_\phi$ a small convolutional network. It thus starts from confidence and departs from it where the counterfactuals disagree, lowering redundant confident cells and raising informative ones. At deployment the budget sets only $K$, so one model serves every budget. 
%

\vspace{-2pt}
\subsection{Feature-Anchored Box Consensus}
\label{sec:boxes}\vspace{-2pt}
The feature budget reaches only the regions where some agent ranks high, so an object outside them can be lost even when a cooperator detects it locally. Intermediate-late hybrids recover such objects by sending detected boxes alongside features, and they differ in which boxes they send and how the boxes contribute. CoSDH~\cite{xu2025cosdh} sends dense pre-NMS local boxes, down-weighted by an empirical weight and fused across agents and against feature-derived boxes by NMS. CoRA~\cite{chen2026cora} sends dense classification and regression maps, realigned by a learned correction module. CoAlign~\cite{lu2023robust} sends the boxes of a separate single-agent detector but uses them only to correct relative poses. CoBEVFlow~\cite{wei2023cobevflow} sends the region-of-interest boxes of its single-agent heads but uses them only to select the features it sends and to move them, by a BEV flow estimated from history boxes, to the receiver's current time.

\noindent\textbf{Box lists.} Each agent in \system appends to its utility map a lightweight \emph{box list}: every box that its onboard detector, the single-agent detector it already runs for its own driving, keeps after non-maximum suppression (NMS). The boxes come from the onboard detector rather than from the cooperative model's single-agent heads, whose features are trained to be warped, fused across agents and decoded jointly (Sec.~\ref{sec:ablation}). For a bounding box, the detector's confidence is the right signal, because it says whether an object is present; the learned utility is reserved for feature cells, whose value is determined by what they add to the other agents' features. 

\noindent\textbf{Why NMS is not enough.} Box lists alone cannot resolve conflicting boxes. Each box carries a pose and a score but none of the semantic evidence behind them, so late fusion by NMS is winner-takes-all: it keeps the highest-scoring box of each object and discards every other observation of it. Adding the fused detections to the same NMS changes only which box wins, and a confident single-view box can still displace a fused box that several agents support (Fig.~\ref{fig:mergequal}b,c).

\noindent\textbf{Feature-anchored box consensus.} We keep the NMS only to decide which detections are retained, but let every observation jointly refine their localization. The fused detections, which carry the features' evidence, enter this NMS and thus \textit{anchor} the set of objects, and a center refinement then reaches a localization consensus. Let $\mathcal{B}$ hold the fused detections and the received boxes, box $j$ with BEV center $\mathbf{p}_j$ and score $q_j$. Each box $b$ that NMS keeps moves to the score-weighted center of the boxes that overlap it,\vspace{-2pt}
\begin{equation}
  \mathbf{p}(b)=\frac{\sum_{j\in\mathcal{N}(b)} q_j\,\mathbf{p}_j}{\sum_{j\in\mathcal{N}(b)} q_j},
  \mathcal{N}(b)=\{j\in\mathcal{B}:\mathrm{IoU}(b,b_j)>\tau_{\text{IoU}}\},
  \label{eq:refine}
\end{equation}
and keeps its size and heading. Unlike weighted boxes fusion~\cite{solovyev2021weighted}, which averages every coordinate of a cluster, consensus moves only the center, so a box seen from one side does not distort the size and heading of the box that NMS kept. An object that several agents see is thus placed by all its observations, while one that a single agent sees keeps its box. The merge needs no learned parameters (Appendix~\ref{app:merge}).

\vspace{-10pt}
\section{Evaluation}
\label{sec:experiments}
\vspace{-4pt}
We ask whether \system reaches the accuracy of methods that send far more data, also on a real C-V2X link (Sec.~\ref{sec:main}); whether the learned utility, rank cap and box consensus each carry part of the result (Sec.~\ref{sec:ablation}); and how it tolerates pose error and packet loss, and at what end-to-end latency (Sec.~\ref{sec:robust}). We provide additional results on scalability and heterogeneous encoders in Appendix~\ref{app:more}.

\begin{table*}[t]
\centering
\caption{Accuracy (AP) and transmission cost (KB; the total data a vehicle broadcasts per frame averaged over vehicles) on V2X-Real~\cite{xiang2024v2xreal} and CooperScene~\cite{wu2026cooperscene}. 
\system is one model per dataset at several budgets ($K$ as a fraction of the map).  Bold: best in each group.}
\label{tab:main}
\vspace{-6pt}
\setlength{\tabcolsep}{5pt}
\footnotesize
\renewcommand{\arraystretch}{0.9}
\begin{tabular}{l rccc c rccc}
\toprule
 & \multicolumn{4}{c}{V2X-Real} && \multicolumn{4}{c}{CooperScene} \\
\cmidrule{2-5}\cmidrule{7-10}
Method & KB & AP@0.3 & AP@0.5 & AP@0.7 && KB & AP@0.3 & AP@0.5 & AP@0.7 \\
\midrule
No fusion & 0 & 53.9 & 52.3 & 42.4 && 0 & 54.4 & 40.7 & 22.5 \\
\midrule
\multicolumn{10}{l}{\emph{(a) Intermediate fusion: features only}} \\
F-Cooper~\cite{f_cooper} & 20,644 & \textbf{65.8} & 55.5 & 36.9 && 18,022 & 76.9 & 57.4 & 21.7 \\
AttFuse~\cite{xu2022opv2v} & 20,644 & 63.2 & 60.9 & 48.0 && 18,022 & 75.6 & 60.5 & 29.8 \\
AF+Enc-Dec~\cite{xu2022opv2v} & 10,322 & 62.9 & 60.9 & 48.5 && 9,011 & 74.9 & 60.1 & 28.3 \\
V2X-ViT~\cite{xu2022v2x} & 5,161 & 61.8 & 59.6 & 46.1 && 4,506 & 82.1 & 70.7 & 38.9 \\
CoST~\cite{tang2025cost} & 5,161 & 62.2 & 60.1 & 43.8 && 4,506 & \textbf{85.3} & \textbf{74.9} & 42.9 \\
ERMVP~\cite{Zhang_2024_CVPR} & 214 & 63.0 & 59.8 & 44.9 && 235 & 79.0 & 62.2 & 29.6 \\
Where2comm~\cite{hu2022where2comm} & 1,335 & 64.3 & 62.3 & 50.2 && 4,474 & 77.7 & 67.1 & 37.6 \\
InfoCom~\cite{wei2025infocom} & 8.1 & 63.3 & 61.3 & 48.4 && 7.2 & 76.9 & 58.3 & 27.8 \\
CodeFilling~\cite{hu2024communication} & 86.2 & 63.8 & 62.6 & \textbf{53.4} && 91.2 & 83.7 & 74.8 & \textbf{45.4} \\
\rowcolor{OursBlue}\system{} {\scriptsize($K$: 0.1\%)} & \textbf{0.26} & 62.5 & 61.1 & 50.2 && \textbf{0.22} & 76.7 & 63.0 & 34.4 \\
\rowcolor{OursBlue}\system{} {\scriptsize($K$: 0.2\%)} & 0.47 & 63.8 & 62.3 & 51.1 && 0.41 & 80.2 & 68.8 & 38.0 \\
\rowcolor{OursBlue}\system{} {\scriptsize($K$: 0.5\%)} & 1.10 & 64.9 & 63.3 & 52.2 && 0.96 & 82.6 & 72.6 & 42.0 \\
\rowcolor{OursBlue}\system{} {\scriptsize($K$: 1\%)} & 2.19 & 65.5 & \textbf{63.8} & 53.0 && 1.89 & 83.5 & 73.8 & 43.6 \\
\rowcolor{OursBlue}\system{} {\scriptsize($K$: 2\%)} & 4.35 & 65.5 & \textbf{63.8} & 53.2 && 3.78 & 83.7 & 74.6 & 43.0 \\
\midrule
\multicolumn{10}{l}{\emph{(b) Box-assisted fusion: features and detected boxes}} \\
CoSDH~\cite{xu2025cosdh} & 352 & 66.7 & 64.9 & 52.9 && 178 & 84.6 & 76.3 & 45.0 \\
CoAlign~\cite{lu2023robust} & 9,032 & 65.8 & 63.4 & 51.0 && 7,885 & 80.3 & 65.3 & 31.1 \\
CoBEVFlow~\cite{wei2023cobevflow} & 197 & 62.8 & 61.4 & 52.1 && 171 & 81.1 & 72.7 & \textbf{48.9} \\
\rowcolor{OursBlue}\system{} {\scriptsize($K$: 0.1\%)} & \textbf{0.45} & 66.8 & 65.3 & 54.1 && \textbf{0.40} & 87.8 & 79.4 & 47.3 \\
\rowcolor{OursBlue}\system{} {\scriptsize($K$: 0.2\%)} & 0.66 & 67.0 & 65.6 & 53.8 && 0.59 & 88.2 & 79.7 & 48.6 \\
\rowcolor{OursBlue}\system{} {\scriptsize($K$: 0.5\%)} & 1.30 & 67.4 & 65.9 & 54.0 && 1.14 & 88.5 & 79.4 & 48.6 \\
\rowcolor{OursBlue}\system{} {\scriptsize($K$: 1\%)} & 2.39 & \textbf{67.8} & \textbf{66.2} & 54.4 && 2.07 & 88.7 & 79.8 & 48.7 \\
\rowcolor{OursBlue}\system{} {\scriptsize($K$: 2\%)} & 4.54 & 67.7 & \textbf{66.2} & \textbf{54.5} && 3.96 & \textbf{88.8} & \textbf{79.9} & 48.6 \\
\bottomrule
\end{tabular}
\vspace{-10pt}
\end{table*}

\vspace{-2pt}
\subsection{Experimental Setup}
\label{sec:exp_setup}\vspace{-2pt}

\noindent\textbf{Datasets and metrics.} We evaluate on two challenging real-world benchmarks: \emph{V2X-Real}~\cite{xiang2024v2xreal}, with two CAVs and two RSUs per scene, and \emph{CooperScene}~\cite{wu2026cooperscene}, with three CAVs, one RSU and \emph{measured C-V2X throughput traces}. Both carry real sensor noise and localization error, whereas older datasets are simulated~\cite{xu2022opv2v,xu2022v2x} or have only two agents~\cite{yu2022dair,xu2023v2v4real}. We report AP ($\uparrow$) at IoU 0.3/0.5/0.7 at the designated ego, and the data each vehicle broadcasts per frame (KB, $\downarrow$).

\noindent\textbf{Baselines.} We compare with dense intermediate fusion (F-Cooper~\cite{f_cooper}, AttFuse with and without encoder--decoder~\cite{xu2022opv2v}, V2X-ViT~\cite{xu2022v2x}, CoST~\cite{tang2025cost}), compression and selection (ERMVP~\cite{Zhang_2024_CVPR}, CodeFilling~\cite{hu2024communication}, InfoCom~\cite{wei2025infocom}, Where2comm~\cite{hu2022where2comm}), and with box-assisted methods that also send detected boxes (CoSDH~\cite{xu2025cosdh}, CoAlign~\cite{lu2023robust}, CoBEVFlow~\cite{wei2023cobevflow}). All are trained on the same splits with the same encoder, and each sender encodes its own point cloud, as a broadcast requires (Appendix~\ref{app:baselines}). 

\noindent\textbf{Implementation and accounting.} All methods use PointPillars~\cite{lang2019pointpillars} with a ResNet BEV backbone~\cite{he2016res} and 64-channel encoder maps at 0.8\,m per cell~\cite{lu2024heal}. \system's perception model is trained for 30 epochs with $\rho^\star{=}0.05$, the codebook for 15 more and the utility estimator, three $3{\times}3$ convolutions with 29\,K parameters, for 4 (Appendix~\ref{app:impl}); one model per dataset serves every budget. Utilities are sent at 4 bits, positions cost 3.5 bits per candidate at $K{=}1\%$, and each transmitted cell carries 36 bits of codebook indices. Every method is billed by the broadcast rule of Sec.~\ref{sec:setup}, including its selection, request and score maps, box lists and pose (Appendix~\ref{app:accounting}). To compare at equal data, we reduce each method's threshold sweep to its Pareto frontier and interpolate AP in log-KB, never extrapolating.

\vspace{-2pt}
\subsection{Main Results}
\label{sec:main}\vspace{-2pt}

\noindent\textbf{Accuracy at a fraction of the data.} \system matches or exceeds methods that send orders of magnitude more data (\cref{fig:teaser,tab:main}). 
Against intermediate fusion, \system's features alone give the best AP@0.5 on V2X-Real from 1.1\,KB on, above CodeFilling at 86\,KB, and on CooperScene come within 0.3 of CoST at about 1,200 times less data; CodeFilling keeps a higher AP@0.7, by 0.2 and 1.8, at 20 and 48 times the data. InfoCom, the only intermediate baseline within CooperScene's 20\,KB link budget, sends seven times the data of our features at $K{=}0.5\%$ yet trails them by 2.0 AP@0.5 on V2X-Real and 14.3 on CooperScene. Against box-assisted methods, the full \system exceeds CoSDH at all three thresholds on both datasets from 0.45 and 0.40\,KB, with 780 and 445 times less data, CoAlign by 1.9 and 14.1 AP@0.5 at about 20,000 times less, and CoBEVFlow by 3.9 and 6.7 AP@0.5 at about 430 times less. With its box lists, \system is the most accurate method of either group at all three thresholds from these budgets on, except on CooperScene at AP@0.7, where CoBEVFlow is 0.2 to 1.6 higher yet with 43 to 430 times more data transmitted. The dense methods and CoAlign, Where2comm on both datasets, and CoSDH on V2X-Real exceed even the capacity of an ideal 27\,Mbps link's 338\,KB per frame, the theoretical upper bound~\cite{ligo2019spectrum}. 

\begin{figure*}[t]
  \centering
  \includegraphics[width=\textwidth]{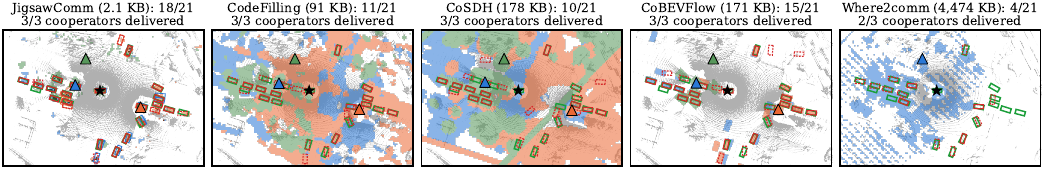}
  \vspace{-14pt}
  \caption{On a real link, \system's message arrives and localizes the most vehicles. Every method's messages are delivered through CooperScene's measured C-V2X throughput traces~\cite{wu2026cooperscene}, at the operating points of Tab.~\ref{tab:cv2x} (\lgd{qGT,line width=0.8pt} ground truth; detections \lgd{qDet,line width=0.8pt} matched at IoU 0.7, \lgd{qDet,line width=0.8pt,dash pattern=on 1pt off 0.6pt} unmatched; $\star$ ego; \lgdtri{qBlue}\,\lgdtri{qOrange}\,\lgdtri{qGreen} cooperators at the pose of the delivered frame; \lgdfill{qBlue,opacity=0.55}\,\lgdfill{qOrange,opacity=0.55}\,\lgdfill{qGreen,opacity=0.55} cells by sender, \lgdfill{qBlue!45!black}\,\lgdfill{qOrange!45!black}\,\lgdfill{qGreen!45!black} rank 2; \lgd{qBlue,line width=0.4pt}\,\lgd{qOrange,line width=0.4pt}\,\lgd{qGreen,line width=0.4pt} \system's received box lists). Titles give the message size and the vehicles localized at IoU 0.7. 
  Appendix~\ref{app:qual} shows more frames on an ideal channel.}
  \label{fig:qual}
  \vspace{-10pt}
\end{figure*}

\begin{table}[t]
\centering
\caption{AP@0.5/0.7 on CooperScene with ideal delivery, CooperScene's default 1.6\,Mbps per link and measured C-V2X traces. \system ($K{=}1\%$) keeps most of its accuracy on a real link. }
\label{tab:cv2x}
\vspace{-6pt}
\setlength{\tabcolsep}{2.6pt}
\footnotesize
\begin{tabular}{l r ccc}
\toprule
Method & KB & Ideal & 1.6\,Mbps & Trace \\
\midrule
V2X-ViT & 4,506 & 70.7/38.9 & 41.0/22.3 & 45.0/26.0 \\
ERMVP & 235 & 62.2/29.6 & 50.7/25.1 & 48.3/24.5 \\
Where2comm & 4,474 & 67.1/37.6 & 44.4/25.2 & 47.4/26.4 \\
CoSDH & 178 & 76.3/45.0 & 59.7/32.8 & 54.8/30.4 \\
CoBEVFlow & 171 & 72.7/\textbf{48.9} & 66.6/40.5 & 57.9/35.3 \\
InfoCom & 7.2 & 58.3/27.8 & 58.3/27.8 & 55.2/26.5 \\
CodeFilling & 91.2 & 74.8/45.4 & 62.8/35.7 & 59.0/32.9 \\
\rowcolor{OursBlue}\system & \textbf{2.07} & \textbf{79.8}/48.7 & \textbf{79.8}/\textbf{48.7} & \textbf{74.5}/\textbf{44.3} \\
\bottomrule
\end{tabular}
\vspace{-10pt}
\end{table}

\noindent\textbf{Accuracy retained under a real C-V2X link.} \system keeps most of its accuracy when messages pass through CooperScene's links, while larger messages lose theirs (Tab.~\ref{tab:cv2x}). \system has no accuracy loss at the benchmark's default constant 1.6\,Mbps and keeps 74.5 of 79.8 AP@0.5 under the measured traces, whose throughput often drops well below the mean. Messages that exceed the link arrive late or not at all: Where2comm, CoSDH and V2X-ViT lose 20--26 AP@0.5 under the traces, and V2X-ViT's dense 4.5\,MB map falls to no fusion at a constant 1.6\,Mbps. CodeFilling, whose score maps alone exceed a frame's capacity, and ERMVP, whose compressed 235\,KB message still does, lose 15.8 and 13.9. InfoCom loses less, but it starts 21.5 AP@0.5 below \system. CoBEVFlow, whose flow moves late features to the ego's current time, loses only 6.1 at 1.6\,Mbps, where each message arrives about 0.9\,s late, but 14.8 under the traces, which deliver its 171\,KB message later still or not at all. \textit{A message must be both small enough to arrive and informative enough to matter}, and under the traces \system leads every baseline by at least 15.5 AP@0.5 (Fig.~\ref{fig:qual}).

\begin{figure}[t]
  \centering
  \includegraphics[width=\linewidth]{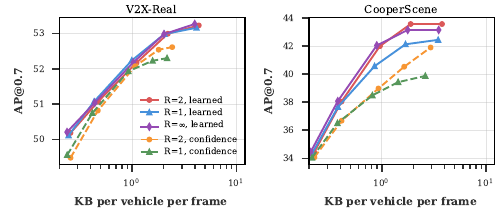}
  \vspace{-22pt}
  \caption{AP@0.7 versus transmitted data (Pareto frontiers) of the features-only model for rank caps $R$ and utilities. A second view helps on CooperScene and is neutral on V2X-Real, a third rarely helps, and the learned utility beats confidence on CooperScene. }
  \label{fig:curves}
  \vspace{-10pt}
\end{figure}

\noindent\textbf{The schedule, not the coding, carries the lead.} CodeFilling's dense score maps cost 70 to 81\,KB per vehicle, beyond CooperScene's link at every threshold. Coded as \system codes its maps (4-bit quantization with entropy-coded positions), they shrink below 1\,KB (dashed in Fig.~\ref{fig:teaser}), yet \system's features alone, without box lists, still lead at small matched budgets, by 4.9 AP@0.5 at 0.5\,KB on V2X-Real and 9.7 at 1\,KB on CooperScene (Appendix~\ref{app:accounting}).


\vspace{-2pt}
\subsection{Ablation Studies}
\label{sec:ablation}\vspace{-2pt}

Each ablation answers one question about a design choice of \system. All but the box-list ablation use the features-only model, which isolates the feature scheduling.

\begin{table}[t]
  \centering
  \caption{AP@0.5/0.7 when the rule in each row picks the senders of each co-observed cell and every other cell is sent by everyone. \emph{Cells}: share of co-observed cells sent, relative to \emph{All observers}. The best observer under the learned utility is worth nearly all observers.} 
  \label{tab:group}
  \vspace{-6pt}
  \footnotesize
  \setlength{\tabcolsep}{3.2pt}
  \begin{tabular}{lccc}
    \toprule
    Co-observed cells sent by & Cells & V2X-Real & CooperScene \\
    \midrule
    Nobody                    & 0      & 52.2/43.9 & 38.1/21.7 \\
    Random observer           & 0.35   & 62.3/51.4 & 65.7/34.8 \\
    Nearest observer          & 0.37   & 62.3/51.4 & 63.7/37.5 \\
    All but the best (learned)& 0.65   & 59.3/50.6 & 57.3/29.0 \\
    Best, confidence          & 0.33   & \textbf{64.1}/53.2 & 70.9/40.6 \\
    Best, learned             & 0.35   & 63.9/\textbf{53.5} & \textbf{73.8}/\textbf{42.9} \\
    \midrule
    All observers             & 1      & 64.0/54.3 & 74.4/42.2 \\
    \bottomrule
  \end{tabular}
  \vspace{-7.5pt}
\end{table}

\noindent\textbf{Does the learned utility schedule better than confidence?} 
Over the whole schedule at matched data (Fig.~\ref{fig:curves}), the learned utility raises AP@0.7 over confidence ranking at $R{=}2$ by 2.9 at 1\,KB and 2.6 at 2\,KB on CooperScene and stays within 0.3 on V2X-Real; deciding only the senders of co-observed cells (Tab.~\ref{tab:group}), it gains 2.3 AP@0.7 on CooperScene and ties on V2X-Real. The estimator learns this decision from local inputs alone: on co-observed cells, it picks the counterfactual target's sender for 56\% of cells on V2X-Real and 72\% on CooperScene, against 38\% and 37\% for confidence (\cref{tab:predict}). The gap between the datasets follows from their views (\cref{tab:viewq}). With three CAVs and one RSU, CooperScene has 58\% of its co-observed vehicles seen by three or more agents, against 32\% for V2X-Real's two CAVs and two RSUs, and more of them lopsided, with a second-best view under a quarter of the best view's LiDAR points. Where the utilities pick different senders, the learned utility picks the best view far more often on both datasets, but this pays off only where views differ: on CooperScene it localizes 2.0 more co-observed vehicles in a hundred at IoU 0.7, 1.6 of them seen by three or more agents, and 0.2 on V2X-Real, where a weaker view usually suffices (Appendix~\ref{app:more}).

\noindent\textbf{Why two senders per cell?} Diminishing returns hold in practice (Fig.~\ref{fig:curves}). Removing all overlap ($R{=}1$) costs up to 1.4 AP@0.7 on CooperScene, where agents see the same vehicles from different sides, so a second view is worth its bits; a third is not, since $R{=}2$ is never more than 0.4 AP@0.7 below sending every candidate. Well-ranked candidates rarely overlap (at a 1\% budget, 11 to 17\% have a better-ranked observer), so the cap costs few bits and bounds the load as more agents observe a region.

\noindent\textbf{Are the utilities comparable across agents?} Tab.~\ref{tab:group} varies which agents send each co-observed cell, while every other cell is sent by everyone. A comparable utility should make its best observer worth nearly as much as all observers, and it does: at a third of the co-observed cells, the best observer under the learned utility comes within 0.8 AP@0.7 of all observers, beating a random observer that sends as many cells by up to 8.1 AP@0.7. At AP@0.5, the nearest observer is no better than a random one, because what an agent sees at a cell depends on occlusion.

\begin{table}[t]
\centering
\caption{Box lists and features complement each other (AP@0.3/0.5/0.7 and KB at $K{=}1\%$). Top: each alone; \emph{box lists only} fuses the onboard detector's lists by NMS. Middle: merge rules, with lists from the cooperative model's single-agent heads; \emph{fused boxes first} adds a received box only where no fused box overlaps it. Bottom: list sources, merged by consensus; \emph{cooperative model} runs its full decoder on the sender's own map.}
\label{tab:hybrid}
\vspace{-6pt}
\setlength{\tabcolsep}{2.2pt}
\footnotesize
\begin{tabular}{l c c c c}
\toprule
 & \multicolumn{2}{c}{V2X-Real} & \multicolumn{2}{c}{CooperScene} \\
\cmidrule(lr){2-3}\cmidrule(lr){4-5}
 & KB & AP & KB & AP \\
\midrule
Features only & 2.19 & 65.5/63.8/53.0 & 1.89 & 83.5/73.8/43.6 \\
Box lists only & 0.22 & 66.2/64.0/47.3 & 0.20 & 87.6/73.3/47.9 \\
\midrule
\multicolumn{5}{l}{\emph{Merge rule (lists from the cooperative model's single-agent heads)}} \\
One NMS & 2.38 & 67.4/65.5/50.1 & 2.08 & 86.1/69.7/45.3 \\
Fused boxes first & 2.38 & 67.2/65.2/53.7 & 2.08 & 85.0/74.8/44.0 \\
Consensus (Ours) & 2.38 & 67.4/65.8/54.0 & 2.08 & 86.6/75.8/44.3 \\
\midrule
\multicolumn{5}{l}{\emph{Local box list source (merged by consensus)}} \\
Cooperative model & 2.35 & 66.7/65.3/53.2 & 2.00 & 85.2/75.7/48.4 \\
\rowcolor{OursBlue}Onboard detector & 2.39 & \textbf{67.8}/\textbf{66.2}/\textbf{54.4} & 2.07 & \textbf{88.7}/\textbf{79.8}/\textbf{48.7} \\
\bottomrule
\end{tabular}
\vspace{-10pt}
\end{table}

\noindent\textbf{What does the feature-anchored box consensus add?} Box lists cover the objects outside the feature budget, while the fused features localize the objects inside it (Tab.~\ref{tab:hybrid}). Box lists alone cost a tenth of the data and find more objects, but each box comes from one view, so on V2X-Real they fall 5.7 AP@0.7 below the features. Fusing them by NMS keeps one box per object (Fig.~\ref{fig:mergequal}): in this frame, the features alone localize 12 of 19 vehicles at IoU 0.7, and late fusion and the NMS hybrid 11 each, whereas feature-anchored box consensus localizes 17 by placing each object synergistically by all the boxes that overlap it. The source of the lists matters too: the cooperative model is trained for fusion and joint detection, not local detection, so lists from the onboard detector add 0.4 and 4.0 AP@0.5 over its single-agent heads' lists, and 2.4 and 6.0 over features alone, for about 0.2\,KB. Our merge gives the best AP@0.5 on both datasets.



\vspace{-4pt}
\subsection{Additional Results}
\label{sec:robust}\vspace{-4pt}

\noindent\textbf{Robustness to pose error and lost packets.} Both real-world datasets are recorded with the vehicles' own positioning, so every result above already includes real pose error, and CooperScene's traces include real link capacity variations. To test harsher conditions, we add Gaussian noise to every vehicle cooperator's pose on V2X-Real~\cite{xu2022v2x,lu2023robust} and drop feature packets at random. Registering poses from the utility maps, at no added data, \system holds 65.7 AP@0.5 up to 1\,m and $1^\circ$ of noise, above CoAlign (61.4); with 30\% of feature packets lost, it loses only 0.6 AP@0.5 on both datasets (Tabs.~\ref{tab:ceil} and~\ref{tab:pose}, Fig.~\ref{fig:robust}; Appendices~\ref{app:more} and~\ref{sec:discussion}).

\begin{figure}[t]
  \centering
  \includegraphics[width=\linewidth]{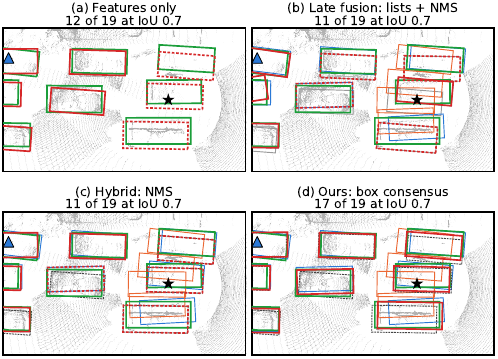}
  \vspace{-14pt}
  \caption{Qualitative visualization of a cropped CooperScene test frame around the ego ($\star$; $\blacktriangle$: a cooperator; inputs: \lgd{qEgo,line width=0.5pt} the ego's own list, \lgd{qBlue,line width=0.5pt}\,\lgd{qOrange,line width=0.5pt} the cooperators' lists, \lgd{qFused,line width=0.5pt,dash pattern=on 1.2pt off 0.6pt} the fused detections at $K{=}1\%$; outputs: \lgd{qDet,line width=0.9pt} matched at IoU 0.7, \lgd{qDet,line width=0.9pt,dash pattern=on 1pt off 0.6pt} unmatched; \lgd{qGT,line width=0.9pt} ground truth). 
  (a) The fused features alone. (b) NMS over the box lists keeps the highest-scoring box, often a single view offset from the object. (c) Adding the fused detections changes which box wins, but one box still places each object. (d) Our feature-anchored box consensus moves each kept box to the score-weighted center of all boxes overlapping it. Counts are over the whole frame.}
  \label{fig:mergequal}
  \vspace{-12pt}
\end{figure}



\vspace{2pt}\noindent\textbf{End-to-end latency.} \system fits the 100\,ms frame of 10\,Hz driving~\cite{ADConstraints,wang2025cp} from capture to detection. On an NVIDIA Jetson AGX Thor in FP16 at $K{=}1\%$, a vehicle computes for 34--36\,ms/frame (Tab.~\ref{tab:latency}): about 10\,ms in its onboard detector, which it runs for its own perception regardless of cooperation, $\sim$20\,ms in the encoder, fusion and heads, and under 6\,ms in the protocol. At CooperScene's 1.6\,Mbps per link, the time from LiDAR capture to fused detection, both broadcasts included, is 45\,ms (Appendix~\ref{sec:discussion}). 

\vspace{-8pt}
\section{Conclusion}
\label{sec:conclusion}
\vspace{-4pt}
We present \system, a cross-agent utility-aware feature scheduling framework for bandwidth-limited cooperative perception. From one broadcast of compact utility maps, every agent derives one schedule, sending each region from its best observer and, budget permitting, second best, by utilities learned from the cooperative model's own counterfactuals, while feature-anchored box consensus over lightweight box lists recovers the objects outside the budget. 
Counting every transmitted byte, \system at 0.45\,KB per vehicle outperforms all baselines on V2X-Real, and under measured C-V2X traces on CooperScene it retains 74.5 of its 79.8 AP@0.5, at least 15.5 points above every baseline.\looseness=-1

\section*{LLM Usage Statement}
We used Anthropic Claude and Google Gemini to improve grammar, clarity, concision, and organization. We also used Anthropic Claude Code for implementation assistance, including code completion, debugging, and refactoring. All LLM-assisted edits and code changes were reviewed and verified by the authors for accuracy, originality, and correctness.

{
    \small
    \bibliographystyle{ieeenat_fullname}
    \bibliography{main}
}

\clearpage
\setcounter{page}{1}
\maketitlesupplementary
\appendix

\section{A Model of Diminishing Returns}
\label{app:utility}
This section motivates the rank cap $R{=}2$ of Sec.~\ref{sec:assembly}: why two senders per cell, rather than one or all? Let $\mathcal{S}^l$ be the agents that send cell $l$, with utilities sorted as $u^l_{[1]}\ge u^l_{[2]}\ge\dots$, and model their combined utility as\vspace{-4pt}
\begin{equation}
  U^l(\mathcal{S}^l)=\sum_{k=1}^{|\mathcal{S}^l|}\gamma_k\,u^l_{[k]}(\mathcal{S}^l),
  \quad 1=\gamma_1\ge\gamma_2\ge\dots\ge\gamma_{|\mathcal{S}^l|}.
  \label{eq:setutility}\vspace{-1pt}
\end{equation}
The best observer counts in full, and each further view contributes diminishing marginal value to the fused result because its evidence overlaps the better ones. If we set $\gamma_{k\ge2}{=}0$, the combined utility simplifies to $U^l=\max_{i\in\mathcal{S}^l}u_i^l$, representing the minimum-mutual-information view of redundancy~\cite{williams2010nonnegative,barrett2015exploration}. Conversely, setting $\gamma_{k\ge2}>0$ credits the synergy of two views of one object from different sides. Because the weights are non-increasing, every additional observer adds less than the one before. Let $\lambda$ be the price of sending one cell, the dual variable of the budget.

\begin{proposition}[Who sends a cell]
\label{prop:lagrange}
Given the candidates of cell $l$, $\max_{\mathcal{S}^l}\,U^l(\mathcal{S}^l)-\lambda|\mathcal{S}^l|$ is attained by including the $k$-th best candidate if and only if $\gamma_k u^l_{[k]}>\lambda$.
\end{proposition}

\noindent The optimal set thus consists of the top candidates, and the $k$-th one adds exactly $\gamma_ku^l_{[k]}$, which does not increase with $k$ (Appendix~\ref{app:proofs}). If $\gamma_k$ vanishes beyond rank $R$ and the price is small, the proposition is the rank cap: a cell is worth sending by its first $R$ observers only. The model treats cells independently, so it serves as an intuitive motivation rather than a strict guarantee for the highly coupled detection task. Sec.~\ref{sec:ablation} measures how much a second view adds and how little a third does.

\section{Proofs}
\label{app:proofs}

Utilities are non-negative and the price is $\lambda\ge0$. The weights satisfy $1=\gamma_1\ge\gamma_2\ge\dots\ge\gamma_n$ for $n=|\mathcal{C}^l|$, the candidates of cell $l$, and may be negative: a negative $\gamma_k$ models a further view that dilutes the fused evidence instead of refining it, as when sending every observer scores below sending fewer (Tab.~\ref{tab:group}). $u_{[k]}(\mathcal{S})$ is the $k$-th largest utility in a set $\mathcal{S}$ of senders, with $u_{[k]}(\mathcal{S})=0$ for $k>|\mathcal{S}|$. Let $k_0$ be the first rank with $\gamma_{k_0}\le0$ ($k_0=n{+}1$ if there is none); since $\gamma$ is non-increasing, $\gamma_k>0$ exactly for $k<k_0$.

\noindent\textbf{Diminishing returns.} Let $T_m(\mathcal{S})=\sum_{k\le m}u_{[k]}(\mathcal{S})$ be the sum of the $m$ largest utilities. Summation by parts gives
\begin{equation*}
U^l(\mathcal{S})=\sum_{m=1}^{n-1}(\gamma_m-\gamma_{m+1})\,T_m(\mathcal{S})+\gamma_n\,T_n(\mathcal{S}).
\end{equation*}
Adding $j$ to $\mathcal{S}$ changes $T_m$ by $\max(0,u_j-u_{[m]}(\mathcal{S}))\ge0$, which is non-increasing in $\mathcal{S}$ because $u_{[m]}(\mathcal{S})$ is non-decreasing, so each $T_m$ is monotone submodular (a known property of top-$m$ sums~\cite{krause2014submodular}). The coefficients $\gamma_m-\gamma_{m+1}$ are non-negative, and $T_n(\mathcal{S})=\sum_{j\in\mathcal{S}}u_j$ is modular, so $U^l$ is submodular for any sign of $\gamma_n$: each further view adds less than the one before. $U^l$ is also monotone when $\gamma_n\ge0$; with negative weights, adding a view can lower it.

\noindent\textbf{Proposition~\ref{prop:lagrange}} (Who sends a cell). \emph{$\max_{\mathcal{S}\subseteq\mathcal{C}^l}U^l(\mathcal{S})-\lambda|\mathcal{S}|$ is attained by sending the $k$-th best candidate iff $\gamma_ku^l_{[k]}>\lambda$.}

\begin{proof}
Write the objective of a set $\mathcal{S}$ with $|\mathcal{S}|=m$ as $\sum_{k\le m}\big(\gamma_ku_{[k]}(\mathcal{S})-\lambda\big)$. For $k<k_0$, $\gamma_k>0$ and $u_{[k]}(\mathcal{S})\le u^l_{[k]}$, the $k$-th largest utility among all candidates, so the term is at most $a_k=\gamma_ku^l_{[k]}-\lambda$. For $k\ge k_0$, $\gamma_k\le0$, so the term is at most $-\lambda\le0$. Hence the objective of $\mathcal{S}$ is at most $\sum_{k\le\min(m,k_0-1)}a_k$. The bounds $a_k$ are non-increasing in $k$ for $k<k_0$, as products of non-increasing positive and non-negative sequences minus a constant, so this sum is at most $\sum_{k\in\mathcal{A}}a_k$ with $\mathcal{A}=\{k:a_k>0\}$, a prefix of the ranks below $k_0$ because $\lambda\ge0$. The best $|\mathcal{A}|$ candidates attain the bound, since for them $u_{[k]}(\mathcal{S})=u^l_{[k]}$. They are therefore a maximizer, and they are the cheapest one: terms with $a_k=0$ may be added without changing the objective, and we exclude them.
\end{proof}

\noindent The proposition needs no sign on $\gamma$: a rank with $\gamma_k\le0$ is never worth sending at any price, so negative weights only make the cap sharper. If $\gamma_k>0$ for $k\le R$ and $\gamma_k\le0$ beyond, every positive candidate up to rank $R$ satisfies the condition as $\lambda\to0^+$, which is the rank cap. Under a per-agent budget, the Lagrangian relaxation prices every cell of agent $i$ at the same $\lambda_i$, so agent $i$ admits its candidates in decreasing order of $\gamma_{r}u$ until the budget is spent, never admitting a candidate whose rank has $\gamma_r\le0$. The $\gamma$ rows of Tab.~\ref{tab:rank} use this order.

\noindent\textbf{Missing utility maps.} \emph{If agent $i$ does not receive some agents' maps, it transmits a superset of the cells it would transmit with all maps, and never more than $K$ cells.} Agent $i$'s rank counts the received candidates that beat it. Removing candidates cannot increase this count, so every cell of true rank $\le R$ keeps a computed rank $\le R$, and all transmitted cells remain within $\mathcal{K}_i$. The statement concerns coverage and assumes that poses are received. Decoding requires a receiver to hold the maps the sender used, which the $N$-bit list in each message makes checkable (Appendix~\ref{sec:discussion}).

\section{Algorithms}
\label{app:alg}

Algorithm~\ref{alg:protocol} gives the protocol that every agent runs in each frame at deployment, and Algorithm~\ref{alg:train} the two training stages. Both use the notation of Sec.~\ref{sec:method}.

\begin{algorithm*}[t]
\caption{\system at agent $i$ in one frame.}
\label{alg:protocol}
\small
\begin{algorithmic}[1]
\Require point cloud $X_i$, pose $T_i$, budget $B$, rank cap $R{=}2$
\State $\mathbf{F}_i \gets \Phi_{\text{enc}}(X_i)$; confidence $\mathbf{c}_i$; point density $\mathbf{d}_i$
\State $u_i^l \gets \sigma\big(\mathrm{logit}(c_i^l)+h_\phi(\mathbf{F}_i,\mathbf{c}_i,\mathbf{d}_i)^l\big)$ for all $l$ \Comment{Eq.~\eqref{eq:residual}}
\State $\mathcal{B}_i \gets$ boxes of the onboard detector on $X_i$ after NMS \Comment{box list}
\State $K \gets$ largest fraction of the map whose message fits $B$
\State $\mathcal{K}_i \gets$ the $K$ cells of highest $u_i^l$ inside some peer's BEV range
\State \textbf{broadcast} $\{(l, \mathrm{Q}_4(u_i^l))\}_{l\in\mathcal{K}_i}$ with coded positions, $T_i$, $\mathcal{B}_i$ \Comment{Publish}
\State \textbf{receive} $(\mathcal{K}_j, u_j, T_j, \mathcal{B}_j)$ from every neighbor $j$
\State warp every $\mathcal{K}_j, u_j$ into agent $i$'s grid with $T_j$ (optionally register $T_j$ from the maps)
\For{every published cell $l$ and every publisher $a$ of $l$}
  \State $r_a^l \gets 1+\big|\{j\ne a: l\in\mathcal{K}_j,\ u_j^l>u_a^l \lor (u_j^l=u_a^l\land j<a)\}\big|$ \Comment{Rank, Eq.~\eqref{eq:rank}}
\EndFor
\State \textbf{broadcast} codebook indices of $\{\mathbf{f}_i^l: l\in\mathcal{K}_i,\ r_i^l\le R\}$ \Comment{Send}
\State \textbf{receive} each sender's features; place them at $\{l: r_j^l\le R\}$ \Comment{no positions sent}
\State $\hat{Y} \gets \Phi_{\text{dec}}\big(\max_j \mathrm{warp}(\mathbf{M}_j\odot\mathbf{F}_j)\big)$ with $\mathbf{M}_i{=}\mathbf{1}$ for the receiver's own map
\State $\mathcal{D} \gets \mathrm{NMS}\big(\hat{Y}\cup\bigcup_j \mathcal{B}_j\big)$ \Comment{hybrid fusion, Sec.~\ref{sec:boxes}}
\State move each box in $\mathcal{D}$ to the score-weighted center of the boxes that overlap it (BEV IoU $>0.3$)
\State \Return $\mathcal{D}$
\end{algorithmic}
\end{algorithm*}

\begin{algorithm*}[t]
\caption{Training \system.}
\label{alg:train}
\small
\begin{algorithmic}[1]
\Require training scenes with labels, target rate $\rho^\star$, weight $\lambda_r$
\Statex \textit{Stage 1: transmission-aware perception model}
\For{each batch of scenes}
  \State encode every agent; score each cell with the gate head, $s_i^l$
  \State $m_i^l \gets$ straight-through Gumbel-softmax competition over agents at $l$ \Comment{Eq.~\eqref{eq:gumbel}}
  \State fuse $\{\mathbf{M}_i\odot\mathbf{F}_i\}$, decode, and update $\theta$ with $\mathcal{L}_{\text{det}}+\lambda_r|\bar\rho-\rho^\star|$ \Comment{Eq.~\eqref{eq:loss}}
\EndFor
\Statex \textit{Stage 1, continued: codebook}
\State freeze $\theta$; fit the residual multi-codebook quantizer~\cite{hu2024communication}
\Statex \textit{Stage 2: utility estimator}
\For{each scene with agents $\mathcal{A}$}
  \State $p(\mathcal{A})$ and $p(\mathcal{A}\setminus\{i\})$ for every $i$ with the frozen model \Comment{$N{+}1$ passes}
  \State $\delta_i^l \gets \max_o |p_o^l(\mathcal{A})-p_o^l(\mathcal{A}\setminus\{i\})|$ \Comment{Eq.~\eqref{eq:counterfactual}}
  \State build cross-agent pairs at co-observed cells and within-agent pairs near the top-$K$ cutoff
  \State update $\phi$ with the weighted pairwise logistic loss on the logits of Eq.~\eqref{eq:residual}
\EndFor
\end{algorithmic}
\end{algorithm*}

\section{Implementation Details}
\label{app:impl}

\noindent\textbf{Architecture.} Each agent runs a PointPillars~\cite{lang2019pointpillars} pillar encoder (64 filters; voxel $0.4{\times}0.4$\,m), a ResNet BEV backbone (3 layers, stride 2, 64 channels) and a ConvNeXt aligner (3 blocks), giving a 64-channel BEV map at 0.8\,m per cell in the agent's own frame. LiDAR ranges are $[\pm100.8, \pm64]$\,m on V2X-Real ($160\times252$ cells) and $[\pm140.8, \pm40]$\,m on CooperScene ($100\times352$). The fusion backbone has three scales (3/5/8 residual layers, 64/128/256 channels). At each scale, the receiver warps received maps into its frame and takes the element-wise maximum, with non-transmitted cells entering as zeros. Following common practice~\cite{hu2022where2comm,hu2024communication,lu2024heal}, the deeper scales are computed from each sender's map and then masked with the down-sampled schedule. Single-agent detection heads, trained on each agent's own labels, provide the confidence $c_i$.

\noindent\textbf{Projection-first input versus feature warping.} 
Both benchmarks build on OpenCOOD~\cite{xu2022opv2v}, whose \texttt{proj\_first} option sets where a cooperator's point cloud is placed, and both support either value natively. Their default is \texttt{proj\_first=True} (\emph{projection-first}): CooperScene's configurations set it explicitly~\cite{wu2026cooperscene}, and V2X-Real's intermediate fusion configurations leave it unset, so its data loader falls back to it~\cite{xiang2024v2xreal}. Every cooperator's point cloud is then projected into the receiver's frame before encoding, so a cooperator's features are computed for one specific receiver. A vehicle could provide such features only by sharing its raw point cloud or by encoding once per receiver, neither of which fits a broadcast in which one message serves every neighbor. We therefore set the benchmarks' own native option to \texttt{proj\_first=False}: each point cloud is encoded in its own sensor frame, as a vehicle does before it broadcasts, and every receiver warps the received BEV features into its own frame with the sender's broadcast pose, at every scale of the fusion backbone. Warping resamples the features on a rotated grid, and a pose error displaces whole cells rather than individual points. Every method in this paper, \system included, is trained and evaluated with feature warping.

\noindent\textbf{Stage 1: perception model.} Adam, learning rate $2{\times}10^{-3}$, weight decay $10^{-4}$, decayed by 10 at epochs 15 and 25, batch size 2; 30 epochs. The scheduler competes all agents of a scene, the receiver included, with the straight-through Gumbel-softmax of Eq.~\eqref{eq:gumbel}. Both temperature and noise scale follow $\eta_e=\max(0.9\cdot0.9^{e},0.01)$ at epoch $e$, so the sampled winner approaches the deterministic one. The threshold $\tau$ is learned with a straight-through gate. The rate term uses $\rho^\star{=}0.05$ and $\lambda_r{=}100$ (two-sided). The detection loss is focal classification, smooth-$L_1$ regression and direction classification, applied to both the fused and the single-agent heads. 

\noindent\textbf{Onboard detector.} The box lists (Sec.~\ref{sec:boxes}) come from a single-agent PointPillars detector with the same pillar encoder and a three-scale backbone (3/5/8 layers, 64/128/256 channels), trained on each vehicle's own point cloud against the objects in its own view, with Adam at learning rate $2{\times}10^{-3}$ for 30 epochs. It shares no weights with \system and is the detector of the box-lists-only ablation (Tab.~\ref{tab:hybrid}) and of the no-fusion row of Tab.~\ref{tab:main}. Each vehicle applies its own NMS and broadcasts every remaining box.

\noindent\textbf{Stage 1, continued: codebook.} CodeFilling's residual multi-codebook quantizer~\cite{hu2024communication} (3 levels, 2 sub-codebooks of 64 entries, 36 bits per cell) is trained for 15 epochs with every other parameter frozen. The frozen modules also keep their BatchNorm statistics, which would otherwise be refitted to the batches of a stage that does not train them. The loss is the detection loss plus a reconstruction loss on the cells that peers transmit, and the epoch is chosen by validation loss.

\noindent\textbf{Stage 2: utility estimator.} $h_\phi$ is a Conv$3{\times}3$(66$\to$32), BN, ReLU, Conv$3{\times}3$(32$\to$32), BN, ReLU and Conv$3{\times}3$(32$\to$1) stack. Its inputs are the agent's BEV feature, its confidence and $\log(1{+}\text{points per cell})$, and its last layer is zero-initialized, so training starts from the confidence ranking. The counterfactual target (Eq.~\eqref{eq:counterfactual}) uses the frozen model at full collaboration and costs $N{+}1$ decoder passes per scene. The ranking loss sums a cross-agent term over all agent pairs at co-observed cells and a within-agent term over 4,096 pairs per agent. Half of the within-agent pairs are drawn from the agent's top-2\% cells by target and half uniformly, and they are contrasted with the estimator's own current top-2\% cells (hard negatives) and uniform cells. Pair weights are $\min(|\delta_a-\delta_b|/\overline{|\delta_a-\delta_b|},5)$, and an $L_2$ penalty of $10^{-4}$ is applied to the logits. We train for 4 epochs with Adam (learning rate $3{\times}10^{-4}$, cosine schedule) and select the epoch by agreement with the target's top-1\% cells on a fixed subset of the validation split.

\noindent\textbf{Why not the stage-1 gate head.}\label{app:e2ehead} The gate head is trained only through the gate $\mathbb{I}[s\ge\tau]$ and the cross-agent competition, so it learns to separate cells above $\tau$ from those below and to order co-observers, but not to order an agent's own cells. It saturates: its output equals 1.0 on about 1\% of the map, so every candidate set of 0.2--2\% of the map contains the same tied cells. Its cheapest operating points are 5.8\,KB on V2X-Real (51.5/43.4 AP@0.5/0.7) and 7.1\,KB on CooperScene (37.6/20.5), and it needs about 12\,KB to match the learned estimator, which reaches higher accuracy at more than 20$\times$ less data.

\noindent\textbf{Inference.} Utilities are quantized to 4 bits ($\mathrm{round}(15u)/15$) before ranking, so every agent ranks the values it receives, and a candidate needs $\mathrm{round}(15u)>0$. Candidate sets are drawn from cells that fall inside another agent's BEV range. Budgets are set by $K$ as a fraction of the map, $\{0.1, 0.2, 0.5, 1, 2\}\%$.

\section{Communication Overhead Accounting}
\label{app:accounting}

\noindent\textbf{Broadcast rule.} Each vehicle broadcasts once per frame and is billed for everything in its message. For \system, these are the 36-bit codebook indices of its transmitted cells, the 4-bit utilities and the context-coded positions of its candidate cells, its box list (16\,B per box: seven box parameters and a score in FP16), and its 192-bit pose. No selection map is sent, because transmitted cells follow from the published maps. The box-lists-only ablation is billed its box list and pose. Dense methods send every cell of the $H{\times}W$ grid at 0.8\,m in FP16: F-Cooper and AttFuse their 256-channel backbone map, whose three scales are upsampled to that grid before fusion, AttFuse with encoder--decoder its 128-channel code, V2X-ViT and CoST a 64-channel map, and CoAlign its 64/128/256-channel maps at its three fusion scales (0.8, 1.6 and 3.2\,m). For selection baselines, the message holds features, the selection map (the cheaper of an $HW$-bit bitmap and a 16-bit index list) and the pose. CoSDH's request maps are billed as well. InfoCom is billed its 256-dimensional FP32 latent, the 4-bit values of its kept mask cells and their positions, which its paper's volume formula leaves out. CodeFilling is billed as its paper describes~\cite{hu2024communication}: every vehicle broadcasts its information score map, a real value for every cell, which we bill at FP16 ($HW\times16$ bits, 80.6\,KB on V2X-Real and 70.4\,KB on CooperScene); every receiver runs the filling-driven selection on the maps it received, and each sender transmits the union of what the receivers select from it. Its selection map is not billed, because the receivers can derive it from the broadcast maps. KB per vehicle is a frame's total divided by its number of vehicles, averaged over frames with at least two agents.

\noindent\textbf{Link model.} At the constant 1.6\,Mbps of Tab.~\ref{tab:cv2x}, messages are delayed by CooperScene's official delay model~\cite{wu2026cooperscene}; under its measured traces, each link carries one message at a time, and frames captured while it is busy are skipped. A cooperator whose first message has not arrived is left out. For messages of several megabytes, the traces deliver nothing in most frames, so the ego falls back to its own view, whereas at 1.6\,Mbps it fuses features about 22\,s old; this is why V2X-ViT and Where2comm score higher under the traces.

\noindent\textbf{Context coder.} The positions of an agent's candidates form a binary $H\times W$ map, coded in raster order. Each cell is coded with a Krichevsky--Trofimov estimator~\cite{krichevsky1981performance} conditioned on its six previously coded neighbors (left, upper-left, up, upper-right, two-left, two-up), which gives 64 adaptive contexts. The code length of an adaptive KT coder depends only on each context's final counts $(n_0,n_1)$, namely $-\log_2\!\big[\Gamma(n_0{+}\tfrac12)\Gamma(n_1{+}\tfrac12)/(\pi\,\Gamma(n_0{+}n_1{+}1))\big]$, so we bill this closed form plus 2 bits for arithmetic-coder termination~\cite{witten1987arithmetic}. At the operating points of Tab.~\ref{tab:main}, it spends 3--7 bits per published cell (3.5 at $K{=}1\%$ on both datasets), against 16 for an index list.

\noindent\textbf{CodeFilling with our map coding.}  To additionally separate \system's schedule from its message coding, we also give CodeFilling \system's coding: its information score maps quantized to 4 bits, the positions of their nonzero cells context-coded, and no selection map. Its maps then cost about 0.3\,KB per vehicle on V2X-Real and 0.6\,KB on CooperScene instead of 80.6 and 70.4\,KB. Matched in data, \system's features alone, without box lists, still lead this variant on V2X-Real by 4.9 AP@0.5 and 3.6 AP@0.7 at 0.5\,KB and by 2.4 and 0.7 at 1\,KB; at 2\,KB they lead by 1.3 AP@0.5 and trail by 0.2 AP@0.7. On CooperScene they lead by 9.7 AP@0.5 and 5.3 AP@0.7 at 1\,KB and by 3.2 and 1.0 at 2\,KB, and the variant has no operating point below its 0.6\,KB map floor. The schedule, not the coding, thus carries \system's lead at small budgets, and the box lists carry it at 2\,KB (Tab.~\ref{tab:hybrid}).

\section{Baselines}
\label{app:baselines}

\noindent\textbf{Training.} On both datasets, every baseline is trained in our pipeline on \system's splits (6,279 training, 717 validation and 2,172 test frames on V2X-Real, and 12,150, 2,406 and 3,936 on CooperScene), on each sender's own frame, with the same LiDAR encoder, Adam (learning rate $10^{-3}$, decayed at epochs 15 and 25), and up to 30 epochs. 
Unless noted below, each baseline is evaluated at its best validation checkpoint, as is \system (for V2X-ViT and ERMVP on CooperScene, epochs 17 and 13). CodeFilling is trained in its two published stages (FP16 features, then the codebook) and evaluated at its default threshold of 0.01 and over a threshold sweep. Where2comm is trained with single-agent supervision of its confidence head; without it, the confidence collapses and Where2comm sends almost nothing (55.9/43.5 AP@0.5/0.7). CoSDH runs with its intermediate-late fusion, as published~\cite{xu2025cosdh}: each vehicle also sends the boxes that its cooperative heads detect on its own features (score at least 0.3), and the receiver scales their scores by 0.8 before one NMS with its fused detections. These boxes add 0.64\,KB per vehicle on V2X-Real and 0.42\,KB on CooperScene and are billed. InfoCom adds its information-bottleneck module to our AttFuse and is fine-tuned from it for 30 epochs with its paper's weight $\beta{=}0.01$; we evaluate its last epoch, because its validation runs at an annealed mask ratio.

\noindent\textbf{CoBEVFlow.} CoBEVFlow~\cite{wei2023cobevflow} is trained in its three published stages: a Where2comm detector whose single-agent heads generate regions of interest (ROIs), a BEV flow estimator trained on ROI tracklets of that detector, and the final fusion detector trained on top of both. Each cooperator sends its 64-channel pillar map inside its ROIs (about 1\% of the map) and its ROI boxes, which are billed at 128 bits per box plus a 32-bit timestamp. The released code does not include the training of the last two stages, so we train the flow estimator with a mean squared error to the ground-truth motion and the fusion detector for 10 epochs (learning rate $2{\times}10^{-4}$), as the paper describes. Training uses its published asynchrony, each message delayed by a binomial number of frames ($n{=}10$, $p{=}0.3$). 
Under the C-V2X link (Tab.~\ref{tab:cv2x}), its latest and earlier messages are the frames the link has delivered, with their true ages, so the flow compensates the delay. The first stage's validation loss does not track its validation AP on either dataset, so its checkpoint is picked by validation AP@0.5 over every fifth epoch (epoch 5 on V2X-Real, 15 on CooperScene). The fusion detector is taken at its lowest validation loss on V2X-Real (epoch 9) and, on CooperScene, where that loss is lowest at its first validation, at its highest validation AP@0.5 (epoch 7).

\noindent\textbf{Box lists only.} The box-lists-only ablation of Tab.~\ref{tab:hybrid}, which amounts to late fusion, uses the onboard detector, trained in our pipeline for 30 epochs. Each vehicle keeps its boxes after its own NMS, without cutting them to any receiver's range, and the receiver runs one NMS over the union of all lists. Sending only each vehicle's $M$ highest-scoring boxes traces its curve: on V2X-Real, $M{=}4$ costs 0.09\,KB for 57.7/43.9 AP@0.5/0.7 and $M{=}8$ costs 0.13\,KB for 60.1/44.8; on CooperScene, 0.09\,KB gives 59.6/38.1 and 0.13\,KB gives 67.3/44.7. Beyond about 12 boxes per vehicle every list is complete, which bounds its accuracy.

\section{Merging Box Lists with Fused Detections}
\label{app:merge}

\begin{figure}[t]
  \centering
  \includegraphics[width=\linewidth]{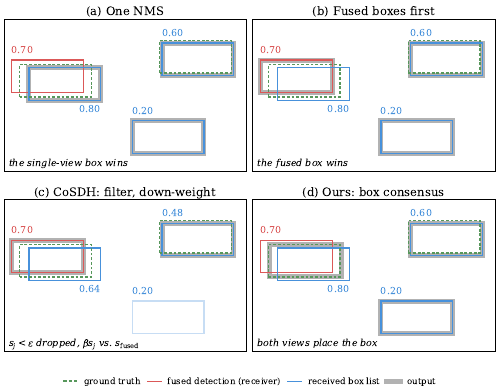}
  \caption{Four ways to merge a receiver's fused detections with received box lists, on an object that the receiver and a cooperator both see, one that only the cooperator sees and a low-scoring false positive. Numbers are detection scores. (a) to (c) keep one box per object, so a single observation places it; (d) keeps NMS to decide which objects survive and places each by all the boxes that overlap it.}
  \label{fig:merge}
\end{figure}

\begin{table}[t]
\centering
\caption{Merging box lists with fused detections: AP@0.3/0.5/0.7 and KB per vehicle at the operating points of Tab.~\ref{tab:main}. Each block fixes the features and the lists and varies only the merge rule (Fig.~\ref{fig:merge}). Lists from the onboard detector hold its boxes after NMS; CoSDH's own lists hold its dense boxes before NMS.}
\label{tab:merge}
\vspace{-6pt}
\setlength{\tabcolsep}{2.2pt}
\footnotesize
\begin{tabular}{l c c c c}
\toprule
 & \multicolumn{2}{c}{V2X-Real} & \multicolumn{2}{c}{CooperScene} \\
\cmidrule(lr){2-3}\cmidrule(lr){4-5}
Merge & KB & AP & KB & AP \\
\midrule
Box lists only, one NMS & 0.22 & 66.2/64.0/47.3 & 0.20 & 87.6/73.3/47.9 \\
\midrule
\multicolumn{5}{l}{\emph{\system features + onboard lists}} \\
One NMS & 2.39 & 67.8/65.8/54.4 & 2.07 & 88.4/77.9/49.1 \\
Fused boxes first & 2.39 & 67.7/65.6/54.0 & 2.07 & 86.2/75.8/44.5 \\
CoSDH's rule & 2.39 & 67.0/65.2/54.3 & 2.07 & 86.3/76.7/46.9 \\
\rowcolor{OursBlue}Consensus (Ours) & 2.39 & 67.8/66.2/54.4 & 2.07 & 88.7/79.8/48.7 \\
\midrule
\multicolumn{5}{l}{\emph{CoSDH features + CoSDH's own lists}} \\
One NMS & 352 & 66.7/64.7/50.4 & 178 & 84.4/74.5/44.7 \\
Fused boxes first & 352 & 66.2/64.3/52.4 & 178 & 84.0/76.3/43.4 \\
CoSDH's rule (published) & 352 & 66.7/64.9/52.9 & 178 & 84.6/76.3/45.0 \\
Consensus (Ours) & 352 & 66.7/65.0/50.8 & 178 & 84.8/73.9/41.2 \\
\bottomrule
\end{tabular}
\vspace{-8pt}
\end{table}

\noindent\textbf{Merge rules.} A receiver holds its fused detections and the box lists of its cooperators, and Fig.~\ref{fig:merge} contrasts four ways to combine them. One NMS over the union (a) keeps the highest-scoring box of each object, so a confident single-view box can displace a fused box that several agents support. Keeping fused boxes first (b) reverses the winner: a received box survives only where no fused box overlaps it, so it adds missed objects but never improves a fused one. CoSDH~\cite{xu2025cosdh} drops received boxes scored below $\epsilon{=}0.3$ and scales the others by $\beta{=}0.8$ before one NMS (c), which favors fused boxes without excluding received ones. All three keep one box per object, so a single observation places each object. Ours (d) keeps the NMS to decide which objects survive and then moves each kept box to the score-weighted center of every box that overlaps it at a BEV IoU above 0.3 ($\tau_{\text{IoU}}=0.3$), fused or received, keeping its size and heading.

\noindent\textbf{Box lists carry recall, features carry localization.} As CoSDH observed for late fusion at low bandwidth~\cite{xu2025cosdh}, box lists alone are a strong baseline: on V2X-Real they match the AP@0.5 of \system's features alone (64.0 against 63.8; Tab.~\ref{tab:hybrid}) for a tenth of the data. Each of their boxes comes from one view, however, and their AP@0.7 falls 5.7 below the features. Neither part reaches the hybrid, which adds 2.4 AP@0.5 over the features and 7.1 AP@0.7 over the lists. The features and box list are hence complementary.

\noindent\textbf{Same lists, different merges.} With \system's features and the onboard lists, our merge gives the highest AP@0.5 on both datasets (66.2 on V2X-Real against 65.2 to 65.8 for the other rules, and 79.8 on CooperScene against 75.8 to 77.9) and stays within 0.4 of the best AP@0.7 (Tab.~\ref{tab:merge}). Keeping fused boxes first and CoSDH's rule, which let the fused box win most overlaps, lose the most AP@0.5.


\noindent\textbf{Dense lists.} CoSDH sends its boxes before NMS, about 40 per vehicle on V2X-Real, whereas an onboard list after NMS holds about 12. Center refinement weighs every overlapping box, so near duplicates from one view outvote the fused box, and AP@0.7 falls 2.1 on V2X-Real and 3.8 on CooperScene below CoSDH's own rule. Refinement assumes one box per object and agent, as lists after NMS provide; we therefore evaluate CoSDH with its published rule everywhere else.

\section{Additional Results}
\label{app:more}

\begin{table}[h]
\centering
\caption{Ablation on the rank cap. AP@0.7 at matched data (0.5/1/2\,KB per vehicle per frame), with the learned utility unless noted. $\gamma$: candidates from a pool twice the budget, sent in decreasing order of $\gamma^{r-1}u$, \ie $\gamma_k{=}\gamma^{k-1}$ in Eq.~\eqref{eq:setutility}. Last row: share of candidates (\%, $K{=}1\%$) with rank $\ge2$ / $\ge3$.}
\label{tab:rank}
\vspace{-6pt}
\setlength{\tabcolsep}{2.4pt}
\footnotesize
\begin{tabular}{l ccc c ccc}
\toprule
 & \multicolumn{3}{c}{V2X-Real} && \multicolumn{3}{c}{CooperScene} \\
\cmidrule{2-4}\cmidrule{6-8}
KB & 0.5 & 1 & 2 && 0.5 & 1 & 2 \\
\midrule
$R{=}1$ & 51.3 & 52.2 & 53.0 && 38.6 & 40.9 & 42.2 \\
\rowcolor{OursBlue}$R{=}2$ & 51.1 & 52.0 & 52.9 && 39.0 & 42.1 & 43.6 \\
$R{=}2$, confidence & 50.9 & 52.0 & 52.6 && 37.2 & 39.2 & 41.0 \\
$R{=}1$, confidence & 51.0 & 52.0 & 52.3 && 37.2 & 38.8 & 39.7 \\
$R{=}\infty$ & 51.2 & 52.2 & 53.0 && 39.4 & 42.2 & 43.1 \\
$\gamma{=}0.5$ & 51.1 & 52.1 & 52.8 && 38.1 & 40.7 & 42.1 \\
$\gamma{=}1$ & 50.9 & 51.8 & 52.6 && 38.2 & 41.5 & 42.9 \\
\midrule
Rank $\ge2/\ge3$ & \multicolumn{3}{c}{11 / 1.7} && \multicolumn{3}{c}{17 / 3.6} \\
\bottomrule
\end{tabular}
\end{table}

\noindent\textbf{Rank cap.} Tab.~\ref{tab:rank} gives the AP@0.7 behind Fig.~\ref{fig:curves} at 0.5, 1 and 2\,KB per vehicle per frame. The $\gamma$ rows draw candidates from a pool twice the budget and send them in decreasing order of $\gamma^{r-1}u$ (Appendix~\ref{app:utility}). The larger pool costs more signaling than the order gains, so both rows stay within 0.3 of $R{=}2$ on V2X-Real and below it on CooperScene, and putting coverage before refinement ($\gamma{=}0.5$) moves AP@0.7 by at most 0.8 in either direction against the plain utility order ($\gamma{=}1$).

\begin{table}[h]
\centering
\caption{Ceilings and packet loss. AP@0.5/0.7 of \system's model with only the ego's own map (\emph{Ego only}; \emph{No fusion} in Tab.~\ref{tab:main} is the onboard detector), with every cell from every agent, and at the $K{=}1\%$ point of Tab.~\ref{tab:main} without and with box lists, on an ideal channel and with each non-ego feature packet lost with probability $p$. The test in \cref{tab:group} builds its masks in the ego frame and warps them back to each sender, so \emph{All observers} differs from \emph{Every cell} in Tab.~\ref{tab:ceil}, which sends every cell in each sender's own grid, by at most 0.3\,AP.}
\label{tab:ceil}
\vspace{-6pt}
\setlength{\tabcolsep}{3pt}
\footnotesize
\begin{tabular}{l cc}
\toprule
 & V2X-Real & CooperScene \\
\midrule
Ego only & 52.2/43.9 & 38.0/21.7 \\
Every cell (133--176\,KB) & 64.0/54.1 & 74.1/42.1 \\
\midrule
Features, $K{=}1\%$ (1.9--2.2\,KB) & 63.8/53.0 & 73.8/43.6 \\
\quad $p{=}0.1$ & 62.9/52.2 & 70.7/41.1 \\
\quad $p{=}0.2$ & 62.0/51.5 & 67.5/38.8 \\
\quad $p{=}0.3$ & 61.0/50.7 & 64.1/36.6 \\
\midrule
\rowcolor{OursBlue}+ box lists (2.1--2.4\,KB) & 66.2/54.4 & 79.8/48.7 \\
\quad $p{=}0.1$ & 66.0/54.4 & 79.6/48.2 \\
\quad $p{=}0.2$ & 65.8/54.3 & 79.4/47.8 \\
\quad $p{=}0.3$ & 65.6/54.3 & 79.2/47.3 \\
\bottomrule
\end{tabular}
\end{table}

\noindent\textbf{Ceilings.} At $K{=}1\%$, \system recovers 98\% of the cooperative AP@0.5 gain on V2X-Real and 99\% on CooperScene with 1.2--1.4\% of the data of sending every cell (Tab.~\ref{tab:ceil}). On CooperScene its AP@0.7 even exceeds that of every cell (43.6 vs.\ 42.1).

\noindent\textbf{Packet loss.} The utility maps and box lists are small, so we keep them and drop each non-ego feature packet independently with probability $p$. Because every cell is sent by few agents, a lost packet removes its cells from the receiver's view. Accuracy degrades gradually (Tab.~\ref{tab:ceil}, Fig.~\ref{fig:robust}, right): at $p{=}0.3$, \system loses 2.8/2.3 AP@0.5/0.7 on V2X-Real and remains 8.8/6.8 above the ego alone, while on CooperScene, where cooperation contributes most, it loses 9.7/7.0 and remains 26.1/14.9 above the ego alone. With box lists, the loss at $p{=}0.3$ shrinks to 0.6/0.1 on V2X-Real and 0.6/1.4 on CooperScene, since the box lists still carry every cooperator's own detections.


\begin{table}[t]
\centering
\caption{How well each utility predicts the counterfactual target $\delta$ (Eq.~\eqref{eq:counterfactual}), on validation scenes. $\rho_{\mathrm{S}}$: Spearman correlation over the cells where an agent's $\delta$ is nonzero. \emph{Who}: share of contested cells (covered by two or more agents) whose highest utility belongs to the agent with the highest $\delta$. \emph{Top-1\%}: overlap between each agent's top 1\% of cells by utility and by $\delta$. The last row also sees the other agents' features, which no sender has before it publishes.}
\label{tab:predict}
\vspace{-6pt}
\setlength{\tabcolsep}{3pt}
\footnotesize
\begin{tabular}{l ccc c ccc}
\toprule
 & \multicolumn{3}{c}{V2X-Real} && \multicolumn{3}{c}{CooperScene} \\
\cmidrule{2-4}\cmidrule{6-8}
Utility & $\rho_{\mathrm{S}}$ & Who & Top-1\% && $\rho_{\mathrm{S}}$ & Who & Top-1\% \\
\midrule
Confidence & 0.06 & 0.38 & 0.25 && 0.14 & 0.37 & 0.42 \\
\rowcolor{OursBlue}Learned (own inputs) & 0.29 & 0.56 & 0.44 && 0.28 & 0.72 & 0.48 \\
With peers' features & 0.61 & 0.60 & 0.65 && 0.65 & 0.76 & 0.63 \\
\bottomrule
\end{tabular}
\vspace{-8pt}
\end{table}

\noindent\textbf{Predictability of the target.} The target $\delta_i^l$ depends on the other agents' features at cell $l$, which the estimator does not see: it reads only the agent's own map, confidence and point density (Eq.~\eqref{eq:residual}). To measure how much of $\delta$ these local inputs explain, we train a second estimator with the same objective that also reads, in its own frame, the maximum of the other agents' features and confidences, a non-causal reference that bounds what peer information adds (Tab.~\ref{tab:predict}). For the decision the schedule makes, which agent sends a contested cell, the local estimator agrees with $\delta$ on 56\% of contested cells on V2X-Real and 72\% on CooperScene, within 4 points of the reference (60\% and 76\%) and far above confidence (38\% and 37\%). Peer information helps mainly to order cells within one agent's budget (top-1\% overlap 0.44 against 0.65 on V2X-Real). Ranking by $\delta$ itself, which needs every agent's full map, gives 63.4/52.6 AP@0.5/0.7 on V2X-Real and 72.9/42.0 on CooperScene at $K{=}1\%$, no better than the learned estimator (63.8/53.0 and 73.8/43.6), so the remaining error in predicting $\delta$ costs no accuracy at this budget. A known weakness of leave-one-out is that two good views of one object can each look dispensable, so neither observer might rank it high. The deployed schedule does not leave such objects out: on a uniform sample of about 600 test frames per dataset at $K{=}1\%$, a vehicle that the ego does not see but two or more cooperators do goes unsent in 1.7\% of cases on V2X-Real and 0.7\% on CooperScene, less often than a vehicle seen by a single cooperator (4.8\% and 6.4\%).

\begin{table}[t]
\centering
\caption{Learned utility versus confidence as the budget grows ($R{=}2$, features only). KB: data per vehicle per frame; AP@0.5/0.7 on the test split. \emph{Shared}: share of the candidates that lie in the top $K$ of two or more agents. \emph{Who}: share of these contested cells whose highest utility belongs to the agent with the highest $\delta$; unlike Tab.~\ref{tab:predict}, it scores only the cells that each utility puts in the top $K$. Shared and Who on validation scenes.}
\label{tab:dense}
\vspace{-6pt}
\setlength{\tabcolsep}{3pt}
\footnotesize
\resizebox{\linewidth}{!}{%
\begin{tabular}{l rccc c rccc}
\toprule
 & \multicolumn{4}{c}{Learned utility} && \multicolumn{4}{c}{Confidence} \\
\cmidrule{2-5}\cmidrule{7-10}
$K$ & KB & AP & Who & Shared && KB & AP & Who & Shared \\
\midrule
\multicolumn{10}{l}{\emph{V2X-Real}} \\
1\%  & 2.19 & 63.8/53.0 & 0.63 & 0.31 && 1.79 & 63.7/52.5 & 0.56 & 0.32 \\
3\%  & 6.46 & 63.8/53.0 & 0.57 & 0.16 && 2.78 & 63.7/52.6 & 0.53 & 0.26 \\
5\%  & 10.6 & 63.8/53.2 & 0.55 & 0.15 && 3.43 & 63.7/52.7 & 0.52 & 0.28 \\
10\% & 20.9 & 63.8/53.4 & 0.51 & 0.20 && 4.83 & 63.7/52.6 & 0.50 & 0.37 \\
\midrule
\multicolumn{10}{l}{\emph{CooperScene}} \\
1\%  & 1.89 & 73.8/43.6 & 0.62 & 0.39 && 1.64 & 72.9/40.5 & 0.62 & 0.36 \\
3\%  & 5.62 & 74.5/43.1 & 0.57 & 0.29 && 3.75 & 74.0/42.2 & 0.57 & 0.34 \\
5\%  & 9.21 & 74.7/42.8 & 0.56 & 0.25 && 4.41 & 74.0/42.2 & 0.55 & 0.33 \\
10\% & 17.7 & 75.1/42.8 & 0.55 & 0.24 && 5.61 & 74.1/42.4 & 0.52 & 0.36 \\
\bottomrule
\end{tabular}}
\vspace{-8pt}
\end{table}

\noindent\textbf{Dense budgets.} As the budget grows from $K{=}1\%$ to $10\%$, both utilities approach the every-cell ceiling (Tab.~\ref{tab:ceil}), and the learned utility's lead in AP@0.7 stays between 0.4 and 0.8 on V2X-Real and shrinks from 3.1 to 0.4 on CooperScene (Tab.~\ref{tab:dense}). Its accuracy on contested cells also declines, from 0.63 to 0.51 Who on V2X-Real and from 0.62 to 0.55 on CooperScene, so where co-observation is dense the schedule moves toward confidence. This Who scores only the cells each utility puts in its top $K$, and so differs from the Who of Tab.~\ref{tab:predict}, which scores every contested cell: confidence ranks its own confident objects first, where it agrees with $\delta$ more often. Over every contested cell, a random choice agrees with $\delta$ less than half the time, since many contested cells have three or four observers. The learned utility still makes the agents' candidates overlap less: from $K{=}3\%$ on, 15\% to 29\% of its candidates are shared, against 26\% to 37\% for confidence, because the learned utility is low where another agent's feature dominates (Sec.~\ref{sec:learning}). Confidence spends every agent's budget on the same objects, whose duplicates the rank cap then removes, so it never sends more than 4.8\,KB on V2X-Real and 5.6\,KB on CooperScene. These leads are at equal $K$, where the learned utility sends more data; at matched data on the Pareto frontiers it stays ahead, with 53.2 against 52.7 AP@0.7 at 4.4\,KB on V2X-Real and 43.6 against 42.2 at 3.8\,KB on CooperScene. The learned utility thus matters most where the budget is scarce, the regime \system targets.


\noindent\textbf{Sender choice.} In Tab.~\ref{tab:group}, a co-observed cell whose best observer is the ego is sent by nobody, since the ego holds it; the random observer is drawn with a fixed seed, and the nearest observer is the agent whose LiDAR is closest to the cell center, which needs only the broadcast poses.

\begin{table}[t]
  \centering
  \caption{Which observer sends a co-observed vehicle, over every vehicle seen by at least two agents in the test splits. \emph{Lopsided}: the second view has less than a quarter of the best view's LiDAR points. The learned utility picks the best view far more often than confidence on both datasets; the choice pays off on CooperScene, where views are more lopsided.}
  \label{tab:viewq}
  \vspace{-6pt}
  \scriptsize
  \setlength{\tabcolsep}{1pt}
  \begin{tabular}{lcc}
    \toprule
    & V2X-Real & CooperScene \\
    \midrule
    Co-observed vehicles                                & 41k   & 59k \\
    \quad seen by ${\ge}3$ agents (\%)                  & 32    & 58 \\
    Utilities pick different senders (\%)               & 43    & 42 \\
    \quad of which best view picked (\%), learned       & 60    & 78 \\
    \quad of which best view picked (\%), confidence    & 24    & 13 \\
    Lopsided vehicles (\%)                              & 37    & 55 \\
    Gain at IoU 0.7, learned over confidence (points)   &       &    \\
    \quad lopsided vehicles                             & 1.3   & 2.6 \\
    \quad all co-observed vehicles                      & 0.2   & 2.0 \\
    \bottomrule
  \end{tabular}
  \vspace{-8pt}
\end{table}

\noindent\textbf{Why learned utilities gain on CooperScene only.} The learned utility gains over confidence on CooperScene but ties it on V2X-Real (Tab.~\ref{tab:group}), and the views each observer has of a vehicle explain why (Tab.~\ref{tab:viewq}). For every vehicle seen by at least two agents, we count each observer's LiDAR points on it and record the sender each utility picks and the fused box. The two utilities disagree equally often on both datasets, and where they do, the learned utility picks the observer with the most points far more often than confidence. The choice pays off only where views differ in quality: lopsided views are more common on CooperScene, and there the better view also helps twice as much. Over all co-observed vehicles, the learned choice localizes 2.0 points more at IoU 0.7 on CooperScene, close to the AP@0.7 gap in Tab.~\ref{tab:group}, and 0.2 points more on V2X-Real, where a weaker view usually suffices. CooperScene also has more observers per vehicle: 58\% of its co-observed vehicles are seen by three or more agents, against 32\% on V2X-Real, and these vehicles give 1.6 of its 2.0 points.


\begin{table}[h]
\centering
\caption{Training rank cap and fusion operator. AP@0.5/0.7 of the features-only model (no box lists) at $K{=}1\%$ and $R{=}2$. \emph{Trained at $R{=}2$} retrains \system's full recipe with two winners per cell in the training competition. The attention rows send exactly \system's messages: attentive fusion~\cite{xu2022opv2v} is swapped into its model at test time, or swapped in and fine-tuned on the receiver side; \emph{masked} leaves the cells a sender did not transmit out of the attention.}
\label{tab:fusion}
\vspace{-6pt}
\setlength{\tabcolsep}{3pt}
\footnotesize
\begin{tabular}{l cc}
\toprule
 & V2X-Real & CooperScene \\
\midrule
\system (max, trained at $R{=}1$) & 63.8/\textbf{53.0} & 73.8/43.6 \\
Trained at $R{=}2$ & \textbf{63.9}/51.3 & \textbf{75.3/45.7} \\
\midrule
Attention, test time & 57.2/48.7 & 40.8/22.9 \\
Masked attention, test time & 56.7/48.5 & 40.0/22.8 \\
Attention, fine-tuned & 63.0/50.4 & 70.5/37.8 \\
Masked attention, fine-tuned & 63.4/50.9 & 70.0/37.6 \\
\bottomrule
\end{tabular}
\end{table}

\noindent\textbf{Training rank cap.} \system trains with one winner per cell in the cross-agent competition (Sec.~\ref{sec:learning}) and deploys with a rank cap of $R{=}2$. Retraining the full recipe with two winners per cell gives mixed results (Tab.~\ref{tab:fusion}): on V2X-Real it matches AP@0.5 and loses 1.7 AP@0.7, and on CooperScene it gains 1.5/2.1 AP@0.5/0.7. The two caps train different things. With a single winner, the task gradient reaches a feature only where it beats every co-observer of its cell, which teaches the model whose cell the fusion needs; a second winner also trains the fusion on the overlapping views it receives at deployment. Since neither cap wins on both datasets, every other result uses the $R{=}1$ model.

\noindent\textbf{Fusion operator.} We replace \system's channel-wise max with AttFuse's per-cell attention~\cite{xu2022opv2v}, in which the receiver's feature at each cell weighs every agent's feature at that cell by scaled dot-product similarity. The attention has no parameters, so we first swap it in at test time, and then fine-tune the receiver side (fusion backbone and detection heads) for 6 epochs while the encoder, utility estimator and codebook stay frozen, so every variant sends the same messages. Because a jigsaw cell arrives from at most two senders and is zero in every other agent's map, we also test a masked attention that leaves untransmitted cells out of the softmax. Max fusion stays the most accurate on both datasets (Tab.~\ref{tab:fusion}). Fine-tuning recovers most of the loss of the test-time swap, and on V2X-Real masked attention comes within 0.4 AP@0.5, yet the fine-tuned variants remain 2.1 to 6.0 points lower at AP@0.7. Attention blends the received feature with the receiver's own, which is often weak exactly where the cell was worth sending, whereas max keeps the stronger of the two; the same property keeps the counterfactual utility target small wherever another agent's feature dominates (Sec.~\ref{sec:learning}).

\begin{table}[t]
\centering
\caption{Heterogeneous fleets on V2X-Real (features only, $K{=}1\%$, AP@0.5/0.7). A second agent type is aligned into the frozen V2X-Real model. \emph{Peers}: the ego keeps PointPillars and every cooperator uses the new encoder; \emph{alt}: the two types alternate; \emph{all}: every agent uses the new encoder. \emph{Fitted to PointPillars}: the utility head of Tab.~\ref{tab:main}; \emph{refit}: the same head retrained on mixed fleets.}
\label{tab:heter}
\vspace{-6pt}
\setlength{\tabcolsep}{3pt}
\footnotesize
\resizebox{\linewidth}{!}{%
\begin{tabular}{ll ccc}
\toprule
 & & \multicolumn{2}{c}{Learned utility} & \\
\cmidrule{3-4}
New encoder & Fleet & Fitted to PointPillars & Refit & Confidence \\
\midrule
SECOND & peers & 62.0/51.5 & \textbf{65.0/54.9} & 64.7/54.3 \\
       & alt   & 62.4/51.9 & \textbf{64.9/54.5} & 64.5/54.0 \\
       & all   & 61.9/51.9 & \textbf{64.4/54.4} & 64.1/53.8 \\
\midrule
LiDAR + camera & peers & 52.3/44.0 & \textbf{63.6/51.4} & 63.2/50.9 \\
               & alt   & 55.7/45.9 & \textbf{63.5/51.3} & 62.9/50.9 \\
               & all   & 50.5/40.5 & \textbf{62.3/47.5} & 61.9/47.2 \\
\bottomrule
\end{tabular}}
\vspace{-8pt}
\end{table}

\noindent\textbf{Heterogeneous agents.} \system assumes that all agents share one feature space, so a new agent type must first be aligned to it. We test two types: SECOND~\cite{yan2018second}, a sparse-convolution LiDAR encoder with $0.1{\times}0.1{\times}0.2$\,m voxels, and a LiDAR and camera encoder in the style of BEVFusion~\cite{liu2023bevfusion}, which adds a lift-splat~\cite{philion2020lift} branch over four cameras (EfficientNet-B0 backbone) to PointPillars. Following HEAL~\cite{lu2024heal}, we train only the new encoder into the frozen V2X-Real model, with its codebook, fusion and heads fixed, on fleets that mix both types. The utility head does not transfer as is (Tab.~\ref{tab:heter}): fitted to PointPillars features, it trails confidence by 1.9 to 2.8 AP@0.7 with SECOND and by 5.0 to 6.9 with the camera encoder, near no fusion (52.3/42.4). Refitting only the utility head on mixed fleets (4 epochs, everything else frozen) restores it: on every fleet it leads confidence by 0.3 to 0.6 AP@0.7 at $K{=}1\%$, and at 0.5\% and 2\% it matches or exceeds confidence on every fleet. At $K{=}1\%$, the refit head also matches or exceeds confidence at $K{=}2\%$, which sends 1.4 to 1.7 times the data. A new agent type thus needs an aligned encoder and a refit utility head, both trained with the rest of the model frozen.

\begin{table*}[t]
\centering
\caption{Scaling with the number of agents on Multi-V2X~\cite{li2024multiv2x}. Models trained at penetration rate 0.3 and tested at every rate with all agents within 70\,m. \emph{Agents}: mean agents per frame, ego included. KB per vehicle per frame and AP@0.5/0.7 on the test split.}
\label{tab:scaling}
\vspace{-6pt}
\setlength{\tabcolsep}{5pt}
\footnotesize
\begin{tabular}{cc c rc rc rc}
\toprule
 & & No fusion & \multicolumn{2}{c}{Where2comm} & \multicolumn{2}{c}{\system, $K{=}0.1\%$} & \multicolumn{2}{c}{\system, $K{=}1\%$} \\
\cmidrule(lr){4-5}\cmidrule(lr){6-7}\cmidrule(lr){8-9}
Rate & Agents & AP & KB & AP & KB & AP & KB & AP \\
\midrule
0.1 & 3.3 & 41.9/30.8 & 7,756 & 68.8/58.6 & 0.40 & 67.8/60.2 & 2.16 & \textbf{69.5/62.3} \\
0.2 & 4.9 & 41.7/31.3 & 8,299 & 71.3/61.0 & 0.42 & 72.3/64.9 & 2.13 & \textbf{73.2/66.3} \\
0.3 & 5.2 & 40.1/30.2 & 8,454 & \textbf{71.4}/63.4 & 0.40 & 70.4/64.2 & 2.10 & 71.3/\textbf{65.7} \\
0.4 & 7.9 & 38.5/28.6 & 8,597 & 72.9/60.6 & 0.41 & 74.9/67.9 & 2.07 & \textbf{76.2/69.6} \\
0.5 & 9.5 & 38.9/29.1 & 8,710 & 77.0/68.3 & 0.42 & 77.4/70.7 & 2.05 & \textbf{78.2/71.8} \\
0.6 & 9.6 & 39.5/29.9 & 8,740 & 73.0/62.8 & 0.41 & 73.9/67.1 & 2.07 & \textbf{74.6/68.0} \\
\bottomrule
\end{tabular}
\vspace{-8pt}
\end{table*}

\noindent\textbf{Scaling with the number of agents.} Both real-world datasets have at most four agents per frame. Multi-V2X~\cite{li2024multiv2x} is simulated, with up to 22 cooperators within 70\,m in our test split, and its penetration rate, the share of connected vehicles, sets how many agents cooperate. We train \system, Where2comm (at its default threshold) and no fusion at rate 0.3 and test them at every rate, from 3.3 to 9.6 agents per frame on average (Tab.~\ref{tab:scaling}). A \system sender's message stays at about 2.1\,KB at every rate, because its budget is a share of its own map and the rank cap bounds how many agents send any cell. With about 3,600 to 4,250 times less data than Where2comm, \system at $K{=}1\%$ leads it by 2.3 to 9.0 AP@0.7 at every rate and matches or exceeds its AP@0.5 at every rate but 0.3 (71.3 against 71.4). At 0.4\,KB ($K{=}0.1\%$), \system still exceeds Where2comm's AP@0.7 at every rate.

\section{Computation Latency}
\label{app:latency}

\begin{table}[t]
\centering
\caption{Computation latency on an NVIDIA Jetson AGX Thor ($K{=}1\%$, $R{=}2$). A vehicle runs the sender stages for its own message and the receiver stages for the messages it receives; per vehicle = receiver + codebook encoding + onboard detector. The detector stages are shared by every intermediate fusion method; the \system stages are the cost of the protocol. The onboard detector produces the vehicle's box list and is the single-agent detector a vehicle runs for its own perception, so the box lists add only the merge. Fusion runs as a TensorRT engine and NMS on the GPU.}
\label{tab:latency}
\vspace{-6pt}
\setlength{\tabcolsep}{3.2pt}
\small
\begin{tabular}{l cc c cc}
\toprule
 & \multicolumn{2}{c}{CooperScene} && \multicolumn{2}{c}{V2X-Real} \\
\cmidrule{2-3}\cmidrule{5-6}
Stage (ms) & FP32 & FP16 && FP32 & FP16 \\
\midrule
\multicolumn{6}{l}{\emph{Detector}} \\
Point cloud encoding & 12.3 & 8.8 && 14.0 & 9.9 \\
Fusion & 8.8 & 5.7 && 10.0 & 6.2 \\
Heads and NMS & 3.8 & 4.1 && 3.9 & 4.1 \\
\midrule
\multicolumn{6}{l}{\emph{\system}} \\
Utility estimation & 1.2 & 1.2 && 1.3 & 1.2 \\
Candidates and utility map & 0.3 & 0.3 && 0.3 & 0.3 \\
Assignment & 0.3 & 0.3 && 0.3 & 0.3 \\
Codebook encoding & 1.0 & 1.0 && 1.0 & 1.0 \\
Codebook decoding & 1.3 & 1.3 && 1.3 & 1.3 \\
\midrule
\multicolumn{6}{l}{\emph{Box lists}} \\
Onboard detector & 15.0 & 9.9 && 16.8 & 10.7 \\
Box merge & 1.5 & 1.6 && 1.5 & 1.5 \\
\midrule
Sender & 30.0 & 21.5 && 33.7 & 23.5 \\
Receiver & 29.5 & 23.3 && 32.6 & 24.8 \\
\rowcolor{OursBlue}Per vehicle & \textbf{45.5} & \textbf{34.1} && \textbf{50.3} & \textbf{36.4} \\
\bottomrule
\end{tabular}
\vspace{-8pt}
\end{table}

\begin{figure*}[t]
  \centering
  \includegraphics[width=\textwidth]{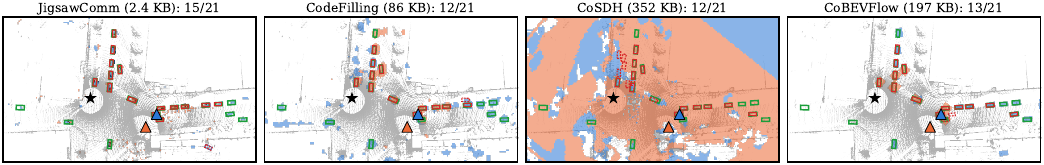}\\[3pt]
  \includegraphics[width=\textwidth]{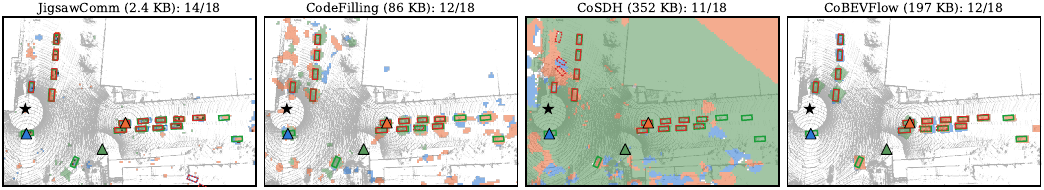}\\[3pt]
  \includegraphics[width=\textwidth]{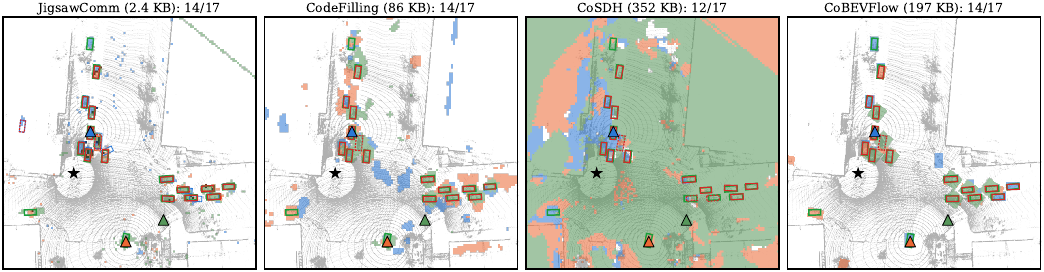}
  \caption{V2X-Real on an ideal channel: \system, CodeFilling, CoSDH and CoBEVFlow (left to right), one test frame per row; titles give the data per vehicle and the vehicles localized at IoU 0.7. Rows: two frames where \system leads and a tie.}
  \label{fig:qual_v2xreal}
\end{figure*}

\begin{figure*}[t]
  \centering
  \includegraphics[width=\textwidth]{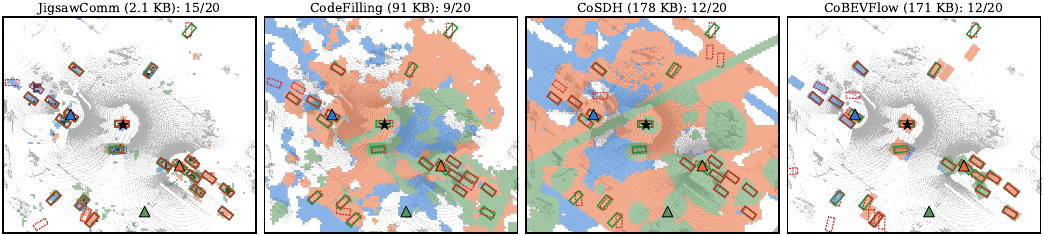}\\[3pt]
  \includegraphics[width=\textwidth]{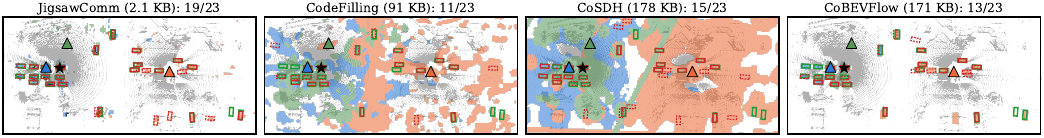}\\[3pt]
  \includegraphics[width=\textwidth]{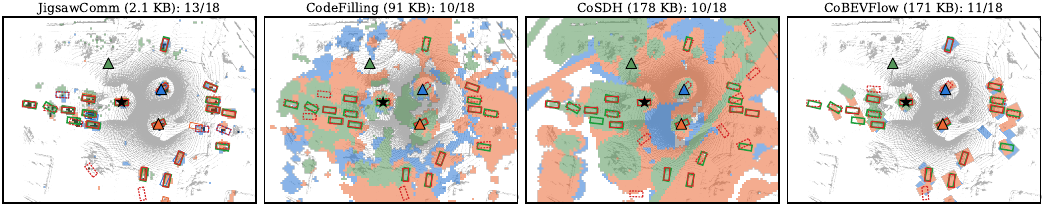}
  \caption{CooperScene on an ideal channel, as in Fig.~\ref{fig:qual_v2xreal}. Rows: three frames where \system leads, by 3, 4 and 2 vehicles.}
  \label{fig:qual_cooperscene}
\end{figure*}

\noindent\textbf{Breakdown.} Tab.~\ref{tab:latency} splits the totals of Sec.~\ref{sec:robust} by stage. Every vehicle runs the same program: it encodes its own point cloud, estimates its utilities, publishes its utility map, derives the assignment, encodes the cells it sends, and decodes, fuses and detects from what it receives; its onboard detector produces the box list it broadcasts, and it merges the lists it receives into its fused detections. Only utility estimation, candidate selection, assignment, the codebook and the box merge belong to the protocol; the onboard detector is the one a vehicle runs for its own perception, and the other stages are the detector that every intermediate fusion method runs.

\noindent\textbf{Measurement.} We measure at the $K{=}1\%$ point of Tab.~\ref{tab:main} on an NVIDIA Jetson AGX Thor, an automotive platform, at its maximum clocks (L4T R39.2, PyTorch 2.10, CUDA 13.0, TensorRT 10.16) with the GPU otherwise idle. A staged implementation reuses the model's modules as a vehicle would run them: each vehicle encodes only its own point cloud, the codebook encodes only the cells the vehicle sends, and the assignment is derived per vehicle from the received utility maps (the scheduler ranks all agents of a scene at once, and we divide its time by the number of agents). Timings are CUDA-synchronized medians over 30 test frames (every 25th frame with at least two agents; 3--4 agents per frame), three runs each after five warm-up frames. Pillar voxelization runs on the CPU and is included in the encoding. The fusion stage runs as a TensorRT engine, candidate selection and assignment as CUDA graphs, and the rotated NMS with OpenPCDet's CUDA kernel on the same top-1,000 candidates as the CPU implementation. Without the box lists, the same pipeline in eager PyTorch with CPU NMS takes 48--75\,ms. In every frame the staged pipeline selects the same cells as the evaluated model and the GPU NMS keeps the same boxes as the CPU NMS; the fused features differ by more than 0.05 on 0.2\% of cells, where the codebook samples a different codeword. A TensorRT point cloud encoder would save another 4--5\,ms but moves a few cells at the top-$K$ boundary, so we do not report it.

\section{Additional Qualitative Results}
\label{app:qual}

Figs.~\ref{fig:qual_v2xreal} and~\ref{fig:qual_cooperscene} compare \system, CodeFilling, CoSDH and CoBEVFlow on an ideal channel, at the operating points of Tab.~\ref{tab:main}, with the drawing conventions of Fig.~\ref{fig:qual}. \system places its cells on the vehicles, each sent by one observer, and its box lists add every cooperator's own detections. CodeFilling, whose senders fill every receiver's information score map, sends 3--4 times as many cells on V2X-Real and about 14 times on CooperScene, many of them on the road surface or on vehicles another agent already covers. CoSDH sends whole fields of view at 86--147 times the data of \system. CoBEVFlow, like \system, sends only cells around objects, inside its regions of interest, but sends each with all 64 channels at full pillar resolution, which costs about 82 to 83 times the data of \system on V2X-Real and CooperScene alike.

\section{Discussion}
\label{sec:discussion}

\noindent\textbf{Properties of the assembly.} Three properties follow from every agent ranking the same maps. \emph{Implicit positions}: a receiver recomputes Eq.~\eqref{eq:rank} for every sender, so the transmitted cells need no selection map. Each feature message carries an $N$-bit list of the maps its sender ranked against (one bit per agent, too few to change any reported number), and a receiver that lacks any of them discards the payload, since it cannot place the cells. \emph{Tolerance to missing maps}: an agent that misses another's map ranks itself against fewer competitors, so its rank can only improve; it may send a duplicate within its own budget, but it never skips a cell it should cover (Appendix~\ref{app:proofs}). A lost feature packet likewise removes only its own cells, and accuracy degrades gradually with the loss rate (Tab.~\ref{tab:ceil}). \emph{Bounded load}: each agent sends at most $K$ cells, and at most $R$ agents send any cell however many observe it. Ranking costs $\mathcal{O}(NK)$ operations per agent, 0.3\,ms on a Jetson AGX Thor (Tab.~\ref{tab:latency}). Both bounds hold up to boundary effects, because each sender ranks on its own grid from bilinearly warped maps; all reported numbers include them.

\noindent\textbf{Training versus deployment.} The perception model is trained under a rule that differs from the deployed one in four ways: $R{=}1$ instead of $R{=}2$, a threshold instead of top-$K$, a single rate target instead of a sweep of budgets, and the gate head instead of the learned utility. The deployed model nonetheless serves every budget with one set of weights (Tab.~\ref{tab:main}) and gains from the second view on CooperScene (Fig.~\ref{fig:curves}), because training teaches it to perceive from sparse input assembled across senders, whichever rule assembled it, and the utility estimator is fitted afterwards to the frozen model.

\noindent\textbf{Practical deployment.} Both messages of \system, the utility map and the feature message, are broadcasts with no addressee, so the protocol needs no coordinator and no feedback. It runs unchanged in V2V, with no infrastructure, and in V2I, with an RSU as one more agent; whichever agent fuses is the ego. The utility map plays for feature-level fusion the role that the object list of the standardized Collective Perception Message plays for late fusion~\cite{cpm}: a compact broadcast from which neighbors decide what is already covered, with learned utilities per cell in place of report recency per object. Each map must reach the others before they send their features, and both messages are far below the link: at 2.07\,KB per frame, a vehicle uses about 10\% of the 20\,KB that CooperScene's benchmark rate of 1.6\,Mbps per link carries in a 100\,ms frame~\cite{wu2026cooperscene}. CooperScene's traces measure every link while all agents broadcast at once (all three links into a receiver carry data in 61\% of frames), so they already include the sharing of the channel; their mean of 0.71\,Mbps per link leaves about 9\,KB per vehicle per frame, still more than four times \system's message, and the trace column of Tab.~\ref{tab:cv2x} replays them link by link. Replaying the two phases at the benchmark rate of 1.6\,Mbps per link, with every stage timed as in Tab.~\ref{tab:latency} and the onboard detector run before the encoder rather than beside it, a sender holds all of its neighbors' maps (541\,B each with the pose and the box list) 23\,ms after capture, ranks and sends its features, and the time from capture to fused detection, all transmission and computation included, is 45\,ms, within 0.2\,ms of a single message of the same bytes. Missing maps cost the features no accuracy: dropping 15\% or 33\% of the maps at random changes AP@0.5/0.7 by at most 0.1 on both datasets and adds 0.2--1.2\% to the features sent, as each sender transmits a superset of its cells. A box list rides with its map, so a receiver that misses a map loses that sender's boxes as well as its features, which it cannot place. With 34--36\,ms of computation in FP16 (Appendix~\ref{app:latency}), signaling, features and fusion fit in one frame. Pose registration from the utility maps (Tab.~\ref{tab:pose}) adds 11.2\,ms per cooperator, about 45\,ms in all with one cooperator and 68\,ms with three (FP16, CooperScene), still within the frame; since it needs only the maps, it can run while the features are in flight. A map that arrives late or not at all costs coverage, not correctness: the sender transmits a superset of its cells, and a receiver without the map discards what it cannot place and fuses the newest message it holds (Appendix~\ref{app:accounting}). As the shared channel fills, C-V2X congestion control lowers each vehicle's rate~\cite{CV2X_ETSI_EN_303_613}; \system follows the lower budget by lowering $K$ alone, and the rank cap keeps the traffic about any region bounded as density grows.

\begin{figure}[t]
  \centering
  \includegraphics[width=\linewidth]{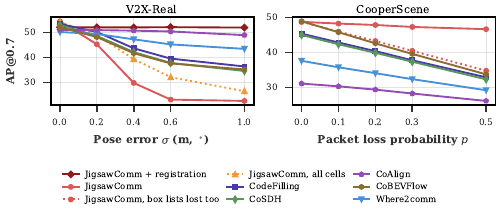}
  \vspace{-16pt}
  \caption{With utility-map registration, \system stays ahead from $\sigma{=}0.2$ on, and it degrades least under packet loss. Left: AP@0.7 on V2X-Real with Gaussian noise of the same $\sigma$ on the position and heading of every vehicle cooperator (Tab.~\ref{tab:pose}). Right: AP@0.7 on CooperScene when each cooperator's feature packet is lost with probability $p$; \system's utility maps and box lists ride in the earlier Publish broadcast, and the dotted line loses each box list with its packet. Every method runs at its Tab.~\ref{tab:main} operating point. \emph{All cells}: the features-only model sending every cell, without box lists.}
  \label{fig:robust}
  \vspace{-8pt}
\end{figure}

\noindent\textbf{Robustness to localization error.} Both real-world datasets are recorded with the vehicles' own positioning, so every reported number already includes real localization error. To measure robustness beyond it, we follow V2X-ViT~\cite{xu2022v2x} and CoAlign~\cite{lu2023robust} and add Gaussian noise to the position and heading of every vehicle cooperator on V2X-Real at test time, while roadside units, fixtures with surveyed poses, stay exact (Tab.~\ref{tab:pose}, Fig.~\ref{fig:robust}, left). Without correction, \system degrades faster than the baselines, because a cooperator's noisy pose misplaces both its features and its box list: at $\sigma{=}1$ it loses 31.8 AP@0.7, against 17.0, 18.3 and 16.9 for CodeFilling, CoSDH and CoBEVFlow and 6.8 for Where2comm, which sends 560 times the data. The utility maps that \system already exchanges also carry the pose, because both agents' maps peak on the objects they share. Registering each cooperator's map against the ego's own, before any feature is sent, adds no data and places both again: \system then holds 65.7 AP@0.5 for every $\sigma$ up to 1, above every baseline at every $\sigma$, and 51.9 AP@0.7 or more, within 2.5 of the clean model and above every baseline from $\sigma{=}0.2$ on. Noising each axis alone shows that position error causes most of the loss: without registration, \system loses 32.6 AP@0.7 at 1\,m and 17.7 at $1^\circ$, and with it, it keeps at least 51.8 at every level on either axis.

\begin{table}[t]
  \centering
  \caption{Utility-map registration makes \system robust to localization error. Gaussian noise with the same $\sigma$ on both the position (m on $x$ and $y$) and the heading (degrees on yaw) of every vehicle cooperator, injected at test time with no retraining; roadside units keep their surveyed poses. AP@0.5/0.7 on the full test split, every method at its Tab.~\ref{tab:main} operating point. \emph{+ registration}: each cooperator's pose registered from the exchanged utility maps, with no added data, and used to place both its features and its box list. \emph{All cells}: the features-only model sending every cell, without box lists.}
  \label{tab:pose}
  \vspace{-6pt}
  \scriptsize
  \setlength{\tabcolsep}{1.5pt}
  \resizebox{\linewidth}{!}{%
  \begin{tabular}{lccccc}
    \toprule
    & \multicolumn{5}{c}{$\sigma$ (m / $^\circ$)} \\
    \cmidrule(lr){2-6}
    Method (KB) & 0 & 0.2 & 0.4 & 0.6 & 1.0 \\
    \midrule
    Where2comm (1,335) & 62.3/50.2 & 61.7/49.4 & 60.4/47.2 & 58.5/45.2 & 55.2/43.4 \\
    CoSDH (352)        & 64.9/52.9 & 63.6/48.6 & 58.7/42.1 & 53.3/37.8 & 46.4/34.6 \\
    CodeFilling (86.2) & 62.6/53.4 & 62.0/50.2 & 59.4/43.7 & 56.0/39.5 & 51.6/36.4 \\
    CoAlign (9,032)    & 63.4/51.0 & 63.3/50.9 & 63.2/50.7 & 62.9/50.4 & 61.4/48.9 \\
    CoBEVFlow (197)    & 61.4/52.1 & 60.7/48.1 & 58.2/41.7 & 55.2/37.6 & 51.3/35.2 \\
    \midrule
    \system, all cells (176) & 64.0/54.1 & 63.5/49.0 & 60.3/39.3 & 55.7/32.1 & 48.0/26.5 \\
    \system{} (2.39)         & \textbf{66.2/54.4} & 65.1/45.3 & 59.5/29.8 & 53.2/23.2 & 46.0/22.6 \\
    \quad + registration (2.39) & 65.7/52.1 & \textbf{65.7/52.0} & \textbf{65.7/52.0} & \textbf{65.7/52.1} & \textbf{65.7/51.9} \\
    \bottomrule
  \end{tabular}}
  \vspace{-8pt}
\end{table}

\noindent\textbf{Registering the utility maps.} The ego holds its own map $U_0$ and each cooperator's map $U_j$, both in their owners' frames and at the 4-bit precision they are broadcast with, together with the broadcast pose $\hat{T}_j$ of agent $j$ in the ego frame. Most cells of a utility map carry a low, nearly uniform value that would dominate any correlation, so each map keeps only its salient cells, $\tilde{U} = \max(U - q_{0.95}(U),\,0)$, where $q_{0.95}$ is the map's 95th percentile. A candidate pose $T = \Delta(x, y, \theta)\,\hat{T}_j$ perturbs the broadcast pose by a planar shift and a yaw in the ego frame; the ego warps $\tilde{U}_j$ into its frame under $T$ by bilinear sampling, which gives $W_T$, and scores the candidate by the zero-mean normalized correlation with $E = \tilde{U}_0$ over the cells $\Omega_T$ that both grids cover,
\begin{equation}
s(T) = \frac{\sum_{\Omega_T} (W_T - \bar{W}_T)(E - \bar{E})}{\sqrt{\sum_{\Omega_T} (W_T - \bar{W}_T)^2 \, \sum_{\Omega_T} (E - \bar{E})^2}},
\end{equation}
with the means $\bar{W}_T$ and $\bar{E}$ taken over $\Omega_T$, so that the score does not reward a candidate for overlapping more of the grid. The search runs in three stages, each centered on the best candidate of the previous one. The first tries shifts of up to 3\,m in steps of 1\,m and yaws of up to $3^\circ$ in steps of $1.5^\circ$ (245 candidates), the next two up to 0.5\,m in steps of 0.25\,m and $0.75^\circ$ in steps of $0.75^\circ$ (75) and up to 0.2\,m in steps of 0.1\,m and $0.25^\circ$ in steps of $0.25^\circ$ (75). The first two stages compare the maps average-pooled by 2 along each axis and the last compares them at full resolution, and each stage scores all its candidates in one batched sampling pass. The best candidate replaces $\hat{T}_j$; height, roll and pitch stay as broadcast, the transform between two cooperators is composed through the ego as $T_{j \leftarrow i} = T_j^{-1} T_i$, and roadside units keep their surveyed poses. Registration costs 0.5/2.3 AP@0.5/0.7 at $\sigma{=}0$, with the recorded poses, and takes 11.2\,ms per cooperator on the Jetson AGX Thor of Appendix~\ref{app:latency}. Without box lists, a search over the same range at full resolution (2,272 candidates in two stages) comes within 0.4 AP@0.7 of the same result at 16 times the time. Registration runs only in the rows of Tab.~\ref{tab:pose} marked with it; every other reported number uses the broadcast poses, which on both datasets already carry the vehicles' real localization error. Because it costs 2.3 AP@0.7 with the recorded poses and gains 6.7 or more from $\sigma{=}0.2$ on, a vehicle would enable it only for a cooperator whose reported position uncertainty exceeds about 0.2\,m, an estimate that GNSS receivers output with every fix.

\end{document}